%% file: arXiv.tex
\documentclass{article} 
\usepackage{arXiv,times}

\input{math_commands.tex}

\usepackage{hyperref}
\usepackage{url}

\usepackage[utf8]{inputenc} 
\usepackage[T1]{fontenc}    
\usepackage{hyperref}       
\usepackage{url}            
\usepackage{booktabs}       
\usepackage{amsfonts}       
\usepackage{nicefrac}       
\usepackage{microtype}      
\usepackage{amsmath}
\usepackage{amsthm}
\usepackage[ruled,vlined]{algorithm2e}
\usepackage{algorithmic}
\usepackage{graphicx}
\usepackage{todonotes}
\usepackage{stackengine}
\usepackage{wrapfig}
\usepackage{tikz}

\title{Regularized Inverse Reinforcement Learning}

\author{Wonseok Jeon\thanks{Correspondence to: Wonseok Jeon \texttt{<jeonwons@mila.quebec>}}~~$^{,1, 2}$,
        Chen-Yang Su$^{1, 2}$,
        Paul Barde$^{1, 2}$,
        Thang Doan$^{1, 2}$,
        \\
        \textbf{Derek Nowrouzezahrai}$^{1, 2}$,
        \textbf{Joelle Pineau}$^{1, 2, 3}$
        \\
        $^{1}$Mila, Quebec AI Institute
        \\
        $^{2}$School of Computer Science, McGill University
        \\
        $^{3}$Facebook AI Research
}

\newcommand{\M}{\mathcal{M}}
\renewcommand{\S}{\mathcal{S}}
\newcommand{\A}{\mathcal{A}}
\newcommand{\X}{\mathcal{X}}

\newcommand{\D}{\mathcal{D}}

\newcommand{\SA}{{\mathcal{S}\times\mathcal{A}}}
\renewcommand{\P}{\Delta}

\newcommand{\I}{\mathbb{I}}
\newcommand{\inner}[3]{\langle#1, #2\rangle_{#3}}
\newtheorem{lemma}{Lemma}
\newtheorem{corollary}{Corollary}
\newtheorem{definition}{Definition}

\newcommand{\pie}{\pi_{\!_E}}
\newcommand{\Qe}{Q_{\!_E}}
\newcommand{\Ve}{V_{\!_E}}
\renewcommand{\re}{r_{\!_E}}
\newcommand{\tildere}{\tilde{r}_{\!_E}}
\newcommand{\De}{\mathcal{D}_{\!_E}}
\newcommand{\bOmega}{\bar{\Omega}}

\renewcommand{\eqref}[1]{Eq.(\ref{#1})}
\newcommand{\lemmaref}[1]{\textbf{Lemma~\ref{#1}}}
\newcommand{\corollaryref}[1]{\textbf{Corollary~\ref{#1}}}
\newcommand{\pd}{\pmb{d}}
\newcommand{\px}{\pmb{x}}
\newcommand{\pD}{\pmb{D}}
\newcommand{\pI}{\pmb{I}}
\newcommand{\one}[1]{\pmb{\mathit{1}}_{#1}}
\newcommand{\ppi}{\bar{\pmb{\pi}}}
\newcommand{\bpi}{\bar{\pi}}

\DeclareMathOperator{\arctantwo}{arctan2}
\newcommand\xrowht[2][0]{\addstackgap[.5\dimexpr#2\relax]{\vphantom{#1}}}
\newcommand{\hpi}{\hat{\pi}}
\newcommand{\htheta}{\hat{\theta}}
\newcommand{\hmu}{\hat{\mu}}
\newcommand{\hSigma}{\hat{\Sigma}}
\newcommand{\hsigma}{\hat{\sigma}}
\newcommand{\hnu}{\hat{\nu}}
\definecolor{darkblue}{rgb}{0.0,0.0,0.0}
\newcommand{\Blue}[1]{{\color{darkblue}#1}}

\usepackage{array}
\newcolumntype{P}[1]{>{\centering\arraybackslash}p{#1}}

\begin{document}

\maketitle
\begin{abstract}
Inverse Reinforcement Learning (IRL) aims to facilitate a learner's ability to imitate expert behavior by acquiring reward functions that explain the expert's decisions.
\emph{Regularized IRL} applies \Blue{strongly} convex regularizers to the learner's policy in order to avoid 
\Blue{the expert's behavior being rationalized by arbitrary constant rewards, also known as degenerate solutions.}
We propose \Blue{tractable} solutions, and practical methods to obtain them, for regularized IRL. 
Current methods are restricted to the maximum-entropy IRL framework, limiting them to Shannon-entropy regularizers, as well as proposing \Blue{the} solutions that are intractable \Blue{in practice}.
We present theoretical backing for our proposed IRL method's applicability for both discrete and continuous controls, empirically validating our performance on a variety of tasks.  
\end{abstract}

\section{Introduction}

Reinforcement learning (RL) has been successfully applied to many challenging domains including games~\citep{mnih2015human,mnih2016asynchronous} and robot control~\citep{schulman2015trust,fujimoto2018addressing,haarnoja2018soft}.
Advanced RL methods often employ policy regularization motivated by, e.g., boosting exploration~\citep{haarnoja2018soft} or safe policy improvement~\citep{schulman2015trust}.
While Shannon entropy is often used as a policy regularizer~\citep{ziebart2008maximum}, \cite{geist2019theory} recently proposed a theoretical foundation of \emph{regularized Markov decision processes} (MDPs)---a framework that uses
\Blue{strongly} convex functions as policy regularizers.
Here, one crucial advantage is that an optimal policy \Blue{is} shown to \emph{uniquely} \Blue{exist}, whereas multiple optimal policies may exist in the absence of policy regularization (and depending on the given reward structure).

Meanwhile, since RL requires a given or known reward function (which can often involve non-trivial reward engineering),
Inverse Reinforce Learning (IRL)~\citep{russell1998learning,ng2000algorithms}---the problem of acquiring a reward function that promotes expert-like behavior---is more generally adopted in practical scenarios like robotic manipulation~\citep{finn2016guided}, autonomous driving~\citep{sharifzadeh2016learning, wu2020efficient} and clinical motion analysis~\citep{li2018inverse}. In these scenarios, defining a reward function beforehand is particularly challenging and IRL is simply more pragmatic.
However, complications with IRL in unregularized MDPs relate to the issue of \Blue{degeneracy, where any constant function can rationalize the expert's behavior~\citep{ng2000algorithms}.} 

Fortunately, \cite{geist2019theory} show that IRL in regularized MDPs---\emph{regularized IRL}---\Blue{does} not contain such degenerate solutions due to the uniqueness of the optimal policy for regularized MDPs.
Despite this, no
\Blue{tractable} solutions of regularized IRL---other than maximum-Shannon-entropy IRL (MaxEntIRL)~\citep{ziebart2008maximum,ziebart2010modeling,ho2016generative,finn2016connection,fu2018learning}---have been proposed.

\Blue{In \citet{geist2019theory}, solutions for regularized IRL were proposed. However,} they are generally intractable since they require
a closed-form relation between the policy and optimal value function and \Blue{the knowledge on model dynamics}. Furthermore, practical algorithms for solving regularized IRL problems have not yet been proposed.

We summarize our \Blue{contributions} as follows:
unlike the solutions in~\cite{geist2019theory}, we propose tractable solutions for regularized IRL problems that \Blue{can be derived from policy regularization and its gradient in discrete control problems (Section~\ref{solutionofirlinregularizedMDP}).
We additionally show that our solutions are tractable for Tsallis entropy regularization with multi-variate Gaussian policies in continuous control problems (Section~\ref{IRLTSALLIS}).}
We devise Regularized Adversarial Inverse Reinforcement Learning (RAIRL), a practical sample-based method for policy imitation and reward learning in regularized MDPs, which generalizes adversarial IRL (AIRL,~\citet{fu2018learning}) (Section~\ref{sec:algorithm}).
Finally, we empirically validate our RAIRL method on both discrete and continuous control tasks, evaluating both episodic scores and divergence minimization \Blue{perspective~\citep{ke2019imitation,ghasemipour2019divergence,dadashi2020primal}} (Section~\ref{experiments}).

\section{Preliminaries}

\paragraph{Notation}
For finite sets $X$ and $Y$, $Y^X$ is a set of functions from $X$ to $Y$.
$\P^X$ ($\P^X_Y$) is a set of (conditional) probabilities over $X$ (conditioned on $Y$). Especially for the conditional probabilities $\Blue{p_{X|Y}}\in\P^X_Y$, we say $\Blue{p_{X|Y}}(\cdot|y)\in\P^X$ for $y\in Y$.
$\R$ is the set of real numbers. For functions $f_1, f_2\in\R^X$, the inner product between $f_1$ and $f_2$ on $X$ is defined as $\inner{f_1}{f_2}{X}:=\sum_{x\in X}f_1(x)\,f_2(x)$.

\paragraph{Regularized Markov Decision Processes and Reinforcement Learning}
We consider sequential decision making problems where an agent sequentially chooses its actions after observing the state of the environment, and the environment in turn emits a reward \Blue{with} state transition. 
Such an interaction between the agent and the environment is modeled as an infinite-horizon \Blue{Markov Decision Process (MDP)}, $\M^r:=\langle\S,\A,P_0,P,r,\gamma\rangle$ and the agent's policy $\pi\in\P_\S^{\A}$. 
The terms within the MDP are defined as follows:
$\S$ is a finite state space,
$\A$ is a finite action space,
$P_0\in\P^\S$ is an initial state distribution,
$P\in\P_{\SA}^\S$ is a state transition probability,
$r\in\R^{\SA}$ is a reward function, and
$\gamma\in[0, 1)$ is the discount factor. 
We also define an MDP without reward as $\M^-:=\langle\S,\A,P_0,P,\gamma\rangle$. 
The normalized state-action visitation distribution, ${d}_\pi\in\Delta^\SA$, associated with $\pi$ is defined as the expected discounted state-action visitation of $\pi$, i.e., 
$
        {d}_\pi(s, a)
        :=
        (1-\gamma)\cdot
        \E_\pi[
            \sum_{i=0}^\infty
            \gamma^i \I\{s_i=s, a_i=a\}
        ]
$,
where the subscript $\pi$ on $\E$ means that a trajectory $(s_0, a_0, s_1, a_1, ...)$ is randomly generated from $\M^-$ and $\pi$, and $\I\{\cdot\}$ is an indicator function. 
Note that ${d}_\pi$ satisfies the transposed Bellman recurrence~\citep{boularias2010bootstrapping,zhang2019gendice}:
\begin{align*}
    d_\pi(s, a)
    =
    (1-\gamma)P_0(s)\pi(a|s)
    +
    \gamma
    \pi(a|s)
    \sum_{\bar{s}, \bar{a}}
    P(s|\bar{s}, \bar{a})d_\pi(s, a).
\end{align*}
We consider RL in regularized MDPs~\citep{geist2019theory}, where the policy is optimized with a causal \Blue{policy} regularizer.
Mathematically for an MDP $\M^r$ and a strongly convex function $\Omega:\P^{\A}\rightarrow\R$, the objective in regularized MDPs is to seek $\pi$ that maximizes the expected discounted sum of rewards, or \emph{return} in short, with policy regularizer $\Omega$:
\begin{align}
    \argmax_{\pi\in\P^\A_\S}J_\Omega(r, \pi)
    :=
    \E_\pi
    \left[
        \sum_{i=0}^\infty
        \gamma^i \{r(s_i, a_i) - \Omega(\pi(\cdot|s_i))\}
    \right]
    =
    \frac{1}{1-\gamma}
    \E_{(s, a)\sim d_\pi}
    \left[r(s, a)-\Omega(\pi(\cdot|s)\right].
    \label{RL_OBJECTIVE}
\end{align}
It turns out that the optimal solution of \eqref{RL_OBJECTIVE} is unique~\citep{geist2019theory}, whereas multiple optimal policies may exist in unregularized MDPs \Blue{(See Appendix~\ref{regularizedrl} for detailed explanation)}. 
In later work~\citep{yang2019regularized}, 
$
    \Omega(\Blue{p})
    =
    -\lambda\E_{a\sim\Blue{p}}\phi(\Blue{p}(a)), \Blue{p}\in\Blue{\P}^\A
$ was considered for $\lambda>0$ and $\phi:(0, 1]\rightarrow\R$ satisfying some mild conditions.
For example, RL with Shannon entropy regularization~\citep{haarnoja2018soft} can be recovered by $\phi(\Blue{x})=-\log\Blue{x}$,
while RL with Tsallis entropy regularization~\citep{lee2020generalized} can be recovered from  $\phi(\Blue{x})=\frac{k}{q-1}(1-\Blue{x}^{q-1})$ for $k>0, q>1$.
The optimal policy $\pi^*$ for \eqref{RL_OBJECTIVE} with $\Omega$ from~\cite{yang2019regularized} is shown to be
\begin{gather}
    \pi^*(a|s)
    =
    \max
    \left\{
        g_{\phi}\left( 
            \frac{\mu^*(s)-Q^*(s, a)}{\lambda}
        \right)
        ,
        0
    \right\},\\
    \label{eq:opt_pol}
    Q^*(s, a)
    =
    r(s, a)+\gamma\E_{s'\sim P(\cdot|s, a)}V^*(s'),
    V^*(s)
    =
    \mu^*(s)
    -
    \lambda\sum_{a\in\A}\pi^*(a|s)^2\phi'(\pi^*(a|s)),
\end{gather}
where $\phi'(x)=\frac{\partial}{\partial x}\phi(x)$, $g_{\phi}$ is an inverse function of $f_{\phi}'$ for $f_{\phi}(x):=x\phi(x), x\in(0, 1]$, and $\mu^*$ is a normalization term such that $\sum_{a\in\A}\pi^*(a|s)=1$.
\Blue{Note that we need to solve constraint optimization problem w.r.t. $\mu^*$ to derive a closed-form relation between optimal policy $\pi^*$ and value function $Q^*$.
However, such closed-form relations have not been discovered except for Shannon-entropy regularization~\citep{haarnoja2018soft} and specific instances ($q=1, 2, \infty$) of Tsallis-entropy regularization~\citep{lee2019tsallis} to the best of our knowledge.}

\vspace{-0.1in}
\paragraph{Inverse Reinforcement Learning}
Given a set of demonstrations from an expert policy $\pie$, IRL~\citep{russell1998learning,ng2000algorithms} is the problem of seeking a reward function from which we can recover $\pi_E$ through RL.
\Blue{However, IRL in unregularized MDPs has been shown to be an ill-defined problem since \emph{(1)} any constant reward function can rationalize every expert and \emph{(2)} multiple rewards meet the criteria of being a solution~\citep{ng2000algorithms}.}
Maximum entropy IRL (MaxEntIRL)~\citep{ziebart2008maximum,ziebart2010modeling} is capable of solving the first issue by seeking a reward function that maximizes expert's return along with Shannon entropy of expert policy. 
Mathematically for the RL objective $J_\Omega$ in~\eqref{RL_OBJECTIVE} and $\Omega=-\mathcal{H}$ for negative
Shannon entropy
$\mathcal{H}(\Blue{p})=\E_{a\sim\Blue{p}}[-\log\Blue{p}(a)]$
~\citep{ho2016generative}, the objective of MaxEntIRL is
\begin{align}
    \mathrm{MaxEntIRL}
    (\pie)
    :=
    \argmax_{r\in\R^\SA}
    \left\{
        J_{-\mathcal{H}}(r, \pie)
        -
        \max_{\pi\in\P_\S^\A} 
        J_{-\mathcal{H}}(r, \pi)
    \right\}.
\label{EQMAXENTIRL}
\end{align}

Another commonly used IRL method is Adversarial Inverse Reinforcement Learning (AIRL)~\citep{fu2018learning} which involves generative adversarial training~\citep{goodfellow2014generative,ho2016generative} to acquire a solution of MaxEntIRL.
AIRL considers the structured discriminator~\citep{finn2016connection} 
$
    D(s, a)
    =
    \sigma(r(s, a) - \log \pi(a|s))
    =
    \frac{e^{r(s, a)}}{e^{r(s, a)}+\pi(a|s)}
$ for $\sigma(x):=1/(1+e^{-x})$ and iteratively optimizes the following objective:
\begin{gather}
    \max_{r\in\R^\SA}~
    \E_{(s, a)\sim{d}_{\pie}}
    \left[ 
        \log D_{r, \pi}(s, a)
    \right] 
    +
    \E_{(s, a)\sim{d}_\pi}
    \left[ 
        \log (1-D_{r, \pi}(s, a))
    \right],
    \nonumber\\
    \max_{\pi\in\P_\S^\A}~
    \E_{(s, a)\sim{d}_\pi}
    \left[
        \log D_{r, \pi}(s, a)
        -
        \log (1-D_{r, \pi}(s, a))
    \right]
    =
    \max_{\pi\in\P_\S^\A}~
    \E_{(s, a)\sim{d}_\pi}
    \left[
        r(s, a)
        -
        \log\pi(a|s)
    \right]
    .\label{eq:airl1}
\end{gather}
It turns out that AIRL minimizes the divergence between visitation distributions ${d}_\pi$ and ${d}_{\pie}$ by solving 
$
    \min_{\pi\in\P_\S^{\A}} \mathrm{KL}(d_\pi || d_{\pie})
$
for Kullback-Leibler (KL) divergence $\mathrm{KL}$~\citep{ghasemipour2019divergence} .

\section{Inverse Reinforcement Learning in Regularized MDPs}
\label{sec:solution_RIRL}

\Blue{In this section, we propose the solution of IRL in regularized MDPs---\emph{regularized IRL}---and relevant properties in}
Section~\ref{solutionofirlinregularizedMDP}.
We then discuss a specific instance of our proposed solution where Tsallis entropy regularizers and multi-variate Gaussian policies are used in continuous action spaces. 

\subsection{Solutions of regularized IRL}
\label{solutionofirlinregularizedMDP}

We consider regularized IRL that generalizes MaxEntIRL in \eqref{EQMAXENTIRL} to IRL with a class of \Blue{strongly} convex policy regularizers:
\begin{align}
    \mathrm{IRL}_{\Omega}(\pie)
    :=
    \argmax_{r\in\R^\SA}
    \left\{
        J_{\Omega}(r, \pie)
        -
        \max_{\pi\in\P_\S^\A} 
        J_{\Omega}(r, \pi)
    \right\}.
    \label{EQRIRL}
\end{align}
For any convex policy regularizer $\Omega$, regularized IRL does not suffer from degenerate solutions since there is a unique optimal policy in any regularized MDP~\citep{geist2019theory}.
While \citet{geist2019theory} proposed solutions of regularized IRL,
\Blue{those are intractable solutions (See Appendix~\ref{intractable_form_solution} for detailed explanation).
In the following lemma, we propose tractable that only requires the evaluation of policy regularizer ($\Omega$) and its gradient ($\nabla\Omega$) which is more tractable in practice.}
Our solution is motivated from figuring out a reward function that is capable of converting regularized RL into equivalent divergence minimization problem \Blue{associated with} $\pi$ and $\pie$:
\begin{lemma}\label{lemma1}
For a policy regularizer $\Omega:\P^\A\rightarrow\R$,
let us define 
\begin{align}
    t(s, a;\pi)
    :=
    \Omega'(s, a;\pi)
    -
    \E_{a'\sim\pi(\cdot|s)}[\Omega'(s, a';\pi)]+\Omega(\pi(\cdot|s))
    \label{TARGET_REWARD}
\end{align}
for $\Omega'(s, \cdot;\pi):=\nabla\Omega(\pi(\cdot|s))
:=[\nabla_{\Blue{p}}\Omega(\Blue{p})]_{\Blue{p}=\pi(\cdot|s)}
\in\R^{\A}, s\in\S$. Then, $t(s, a;\pie)$ for expert policy $\pie$ is a solution of regularized IRL with $\Omega$.
\end{lemma}
\begin{proof} (Abbreviated. See Appendix~\ref{proofoflemma1} for full version.)
With $r(s, a) = t(s, a;\pie)$, the RL objective in \eqref{RL_OBJECTIVE} becomes equivalent to \emph{a problem of minimizing the discounted sum of Bregman divergences between $\pi$ and $\pie$}
\begin{align}
    \argmin_{\pi\in\P_\S^\A}
    \E_\pi
    \left[
        \sum_{i=0}^\infty
        \gamma^i
        D_\Omega^{\A}(\pi(\cdot|s_i)||\pie(\cdot|s_i))
    \right],
    \label{BREGMAN_OBJECTIVE}
\end{align}
where $D_\Omega^{\A}$ is the Bregman divergence ~\citep{bregman1967relaxation} defined by 
$
D_\Omega^{\A}(\Blue{p_1}||\Blue{p_2})
=
\Omega(\Blue{p_1})
-
\Omega(\Blue{p_2})
-
\inner{\nabla\Omega(\Blue{p_2})}{\Blue{p_1}-\Blue{p_2}}{\A}
$
for $\Blue{p_1}, \Blue{p_2}\in\Delta^{\A}$.
Due to the non-negativity of the Bregman divergence, $\pi=\pie$ is a solution of \eqref{BREGMAN_OBJECTIVE} and is unique since \eqref{RL_OBJECTIVE} has the unique solution for arbitrary reward functions~\citep{geist2019theory}. 
\end{proof}
Especially for any policy regularizer $\Omega$ represented \Blue{by} an expectation over the policy~\citep{yang2019regularized}, \lemmaref{lemma1} can be reduced to the following solution in \corollaryref{corollary1}:
\begin{corollary}\label{corollary1} 
For
$
    \Omega(\Blue{p})
    =
    -\lambda\E_{a\sim\Blue{p}}\phi(\Blue{p}(a))
$ 
with
$    
    \Blue{p}\in\Delta^{\A}
$
~\citep{yang2019regularized}
,
\eqref{TARGET_REWARD} becomes
\begin{align}
    t(s, a; \pi)
    =
    -\lambda\cdot
    \left\{
        f_{\phi}'(\pi(a|s))
        -
        \E_{a'\sim\pi(\cdot|s)}
        [
            f_{\phi}'(\pi(a'|s))
            -
            \phi(\pi(a'|s))
        ]
    \right\}
    \label{REFORMED_TARGET_REWARD}
\end{align}
for $f_\phi'(x)=\frac{\partial}{\partial x}(x\phi(x))$. \Blue{The proof is in  Appendix~\ref{proofoflemma1}.}
\end{corollary}
\vspace{-0.05in}
Throughout the paper, we denote \textbf{reward baseline} by the expectation 
$
\E_{a\sim\pi(\cdot|s)}
[
    f_{\phi}'(\pi(a|s))
    -
    \phi(\pi(a|s))
]
$.
\Blue{Note that for} \emph{continuous control tasks} with 
$
    \Omega(\Blue{p})
    =
    -\lambda\E_{a\sim\Blue{p}}\phi(\Blue{p}(a))
$, we can obtain the same form of the reward in~\eqref{REFORMED_TARGET_REWARD} (The proof is in Appendix~\ref{proofforcontinuouscontrol}). \Blue{Although the reward baseline is generally not a intractable in continuous control tasks, we derive a tractable reward baseline for a special case (See Section~\ref{IRLTSALLIS}).} 
Additionally, when $\lambda=1$ and $\phi(x)=-\log x$, it can be shown that $t(s, a;\pi)=\log\pi(a|s)$\Blue{---for both discrete and continuous control problems---}which was used as a reward objective in the previous work~\citep{fu2018learning}, and that the Bregman divergence in \eqref{BREGMAN_OBJECTIVE} becomes the KL divergence $\mathrm{KL}(\pi(\cdot|s)||\pie(\cdot|s))$. 

Additional solutions of regularized IRL can be found by shaping $t(s, a; \pie)$ as stated in the following lemma:
\begin{lemma}[Potential-based reward shaping]
\label{lemma2}
Let $\pi^*$ be the solution of \eqref{RL_OBJECTIVE} in a regularized MDP $\M^r$ with a regularizer $\Omega:\Delta^{\A}\rightarrow\R$ and a reward function $r\in\R^\SA$.
Then \Blue{for $\Phi\in\R^\S$, using either $r(s, a) +\gamma \Phi(s')- \Phi(s)$ or $r(s, a) +\E_{s'\sim P(\cdot|s, a)}\Phi(s') - \Phi(s)$ as a reward} does not change the solution of \eqref{RL_OBJECTIVE}. 
\Blue{The proof is in Appendix~\ref{proofoflemma2}.}
\end{lemma}

From \lemmaref{lemma1} and \lemmaref{lemma2}, we prove the sufficient condition of rewards being solutions of the IRL problem. However, the necessary condition---a set of those solutions are the only possible solutions for an arbitrary MDP---is not proved~\citep{ng1999policy}.

\Blue{In the following lemma, we check how the proposed solution}
is related to the normalized state-visitation distribution which can be discussed in the line of distribution matching perspective on imitation learning problems~\citep{ho2016generative,fu2018learning,ke2019imitation,ghasemipour2019divergence}:

\begin{lemma}\label{lemma3}
Given the policy regularizer $\Omega$,
let us define
$
    \bOmega(d)
    :=
    \E_{(s, a)\sim d}[\Omega(\bar{\pi}_d(\cdot|s))]
$
for an arbitrary normalized state-action visitation distribution $d\in\Delta^\SA$ and the policy $\bar{\pi}_d(a|s):=\frac{d(s, a)}{\sum_{a'\in\A}d(s, a')}$ induced by $d$. Then, \eqref{TARGET_REWARD} is equal to
\begin{align}
    t(s, a;\bar{\pi}_d)
    =
    [\nabla\bOmega(d)](s, a).
    \label{EQSOLUTIONREW}
\end{align}
When $\bOmega(d)$ is strictly convex and a solution
$
    t(s, a;\pie)
    =
    \nabla\bOmega(d_{\pie})
$ of IRL in \eqref{EQSOLUTIONREW} is used, the RL objective in $\eqref{RL_OBJECTIVE}$ is equal to
\begin{align*}
    \argmin_{\pi\in\P^\A_\S} D_{\bOmega}^{\SA}(d_\pi||d_{\pie}),
\end{align*}
where $D_{\bOmega}^{\SA}$ is \Blue{the} Bregman divergence among visitation distributions defined by
$
    D_{\bOmega}^{\SA}(d_1||d_2)
    =
    \bOmega(d_1)
    -
    \bOmega(d_2)
    -
    \langle
        \nabla\bOmega(d_2),
        d_1-d_2
    \rangle
$
for visitation distributions $d_1$ and $d_2$.
The proof is in Appendix~\ref{proofoflemma3}.
\end{lemma}
\vspace{-0.05in}
\Blue{Note that the strict convexity of $\bOmega$ is required for $D_{\bOmega}^{\SA}$ to become a valid Bregman divergence.}
Although the strict convexity of a policy regularizer $\Omega$ does not guarantee \Blue{the} assumption on the strict convexity of $\bOmega$, it has been shown to be true for \Blue{Shannon entropy regularizer ($\bOmega=-\bar{H}$ of \textbf{Lemma 3.1} in~\citet{ho2016generative}) and} Tsallis entropy \Blue{regularizer with its constants $k=\frac{1}{2}, q=2$ ($\bOmega=-\bar{W}$ of \textbf{Theorem 3} in~\citet{lee2018maximum}).}

\subsection{IRL with Tsallis entropy regularization and Gaussian policies}
~\label{IRLTSALLIS}
\vspace{-0.1in}

For continuous controls, multi-variate Gaussian policies are \Blue{often} used in practice~\citep{schulman2015trust,schulman2017proximal} and we consider IRL problems with those policies in this subsection. In particular, we consider IRL with Tsallis entropy regularizer 
$\Omega(\Blue{p})=-\mathcal{T}_q^k(\Blue{p})=-\E_{a\sim\Blue{p}}[\frac{k}{q-1}(1-\Blue{p(a)}^{q-1})]$~\citep{lee2018maximum,yang2019regularized,lee2020generalized}
for a multi-variate Gaussian policy $\pi(\cdot|s)=\mathcal{N}\left(\pmb{\mu}(s), \pmb{\Sigma}(s)\right)$ with $\pmb{\mu}(s)=[\mu_1(s), ..., \mu_d(s)]^T$,  $\pmb{\Sigma}(s)=\mathrm{diag}\{(\sigma_1(s))^2, ..., (\sigma_d(s))^2\}$. 
In such a case, we can obtain \Blue{tractable} forms of the following quantities:

\textbf{Tsallis entropy.} 
The \Blue{tractable} form of Tsallis entropy for a multi-variate Gaussian policy is
\begin{align*}
    \mathcal{T}_q^k(\pi(\cdot|s))
    =
    \frac{
        k(1-e^{(1-q)\mathcal{R}_q(\pi(\cdot|s))})
    }{
        q-1
    },
    \mathcal{R}_q(\pi(\cdot|s))
    =
    \sum_{i=1}^d
    \left\{
        \log(\sqrt{2\pi}\sigma_i(s))-\frac{\log q}{2(1-q)}
    \right\}
\end{align*}
for Renyi entropy $\mathcal{R}_q$. Its derivation is given in Appendix~\ref{appendix:tsallis}. 

\textbf{Reward baseline.}
\Blue{The} reward baseline term 
$
\E_{a\sim\pi(\cdot|s)}
    [
        f_{\phi}'(\pi(a|s))
        -
        \phi(\pi(a|s))
    ]
$
in~\corollaryref{corollary1} is generally intractable except for either discrete control problems or \Blue{Shannon entropy regularization where the reward baseline is equal to $-1$.}
Interestingly, as long as the \Blue{tractable} form of Tsallis entropy can be derived, that of the corresponding reward baseline can also be derived since the reward baseline satisfies
\begin{align*}
    \E_{a\sim\pi(\cdot|s)}
    [
        f_{\phi}'(\pi(a|s))
        -
        \phi(\pi(a|s))
    ]
    =
    (q-1)
    \E_{a\sim\pi(\cdot|s)}
    [
        \phi(\pi(a|s)
    ]
    -k
    =(q-1)\mathcal{T}_q^k(\pi(\cdot|s))-k.
\end{align*}
Here, the first equality holds with $f_\phi'(x)=\frac{k}{q-1}(1-qx^{q-1})=q\phi(x)-k$ for Tsallis entropy regularization.

\textbf{Bregman divergence associated with Tsallis entropy.}
For two different multivariate Gaussian policies, we derive the \Blue{tractable} form of the Bregman divergence (associated with Tsallis entropy) between two policies. The resultant divergence has a complicated form, so we leave it in Appendix~\ref{bregman_tsallis} with its derivation.

\begin{center}
    \begin{figure}[h]
        \vspace{-10pt}
        \centering
        \includegraphics[width=\textwidth]{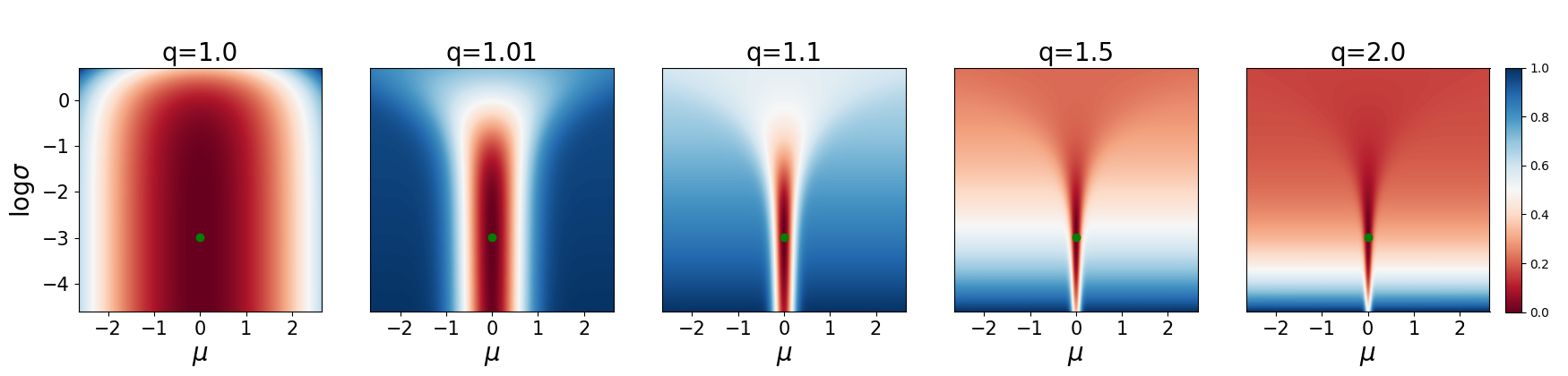}
        \vspace{-20pt}
        \caption{
        Bregman divergence $D_\Omega(\pi||\pie)$ associated with Tsallis entropy ($\Omega(\Blue{p})=-\mathcal{T}_q^1(\Blue{p})$)
        between two uni-variate Gaussian distributions 
        $\pi=\mathcal{N}(\mu, \sigma^2)$
        and    
        $\pie=\mathcal{N}(0, (e^{-3})^2)$ (green point in each subplot). In each subplot, we normalized the Bregman divergence so that the maximum value becomes 1. Note that for $q=1$, $D_\Omega(\pi||\pie)$ becomes the KL divergence $\mathrm{KL}(\pi||\pie)$.
        }
        \label{fig:bregmandivergencegaussian}
        \vspace{-20pt}
    \end{figure}
\end{center}
For deeper understanding of Tsallis entropy and its associated Bregman divergence in regularized MDPs, we consider an example in Figure~\ref{fig:bregmandivergencegaussian}. We first assume that both learning agents' and experts' policies follow uni-variate Gaussian distributions $\pi=\mathcal{N}(\mu, \sigma^2)$ and $\pie=\mathcal{N}(0, (e^{-3})^2)$, respectively. We then evaluate the Bregman divergence \Blue{in Figure~\ref{fig:bregmandivergencegaussian}} by using \Blue{its} \Blue{tractable} form \Blue{and} varying $q$ from 1.0---which corresponds to the KL divergence---to 2.0. We observe that the constant $q$ from the Tsallis entropy affects the sensitivity of the associated Bregman divergence w.r.t. the mean and standard deviation of
the learning agent's policy $\pi$.
Specifically, as $q$ increases, the size of the valley--- relatively red regime in Figure~\ref{fig:bregmandivergencegaussian}---across the $\mu$-axis and $\log\sigma$-axis decreases. This suggests that for larger $q$, minimizing the Bregman divergence requires more tightly matching means and variances of $\pi$ and $\pie$.




\vspace{-0.1in}
\section{Algorithmic Consideration}
\label{sec:algorithm}
\begin{center}
\begin{algorithm}[t]
\caption{Regularized Adversarial Inverse Reinforcement Learning (RAIRL)}
\label{algorithm:reg_airl}
\begin{algorithmic}[1]
\STATE 
\textbf{Input:} 
A set $\De$ of expert demonstration generated by expert policy $\pie$, a reward approximator $r_\theta$, a policy $\pi_\psi$ for neural network parameters $\theta$ and $\psi$
\FOR{each iteration}
    \STATE
    Sample rollout trajectories by using the learners policy $\pi_\psi$.
    \STATE
    Optimize $\theta$ with the discriminator $D_{r_\theta, \pi_\psi}$ and the learning objective in~\eqref{eq:obj_rew}.
    \STATE
    Optimize $\psi$ with Regularized Actor Critic by using $r_\theta$ as a reward function.
\ENDFOR
\STATE
\textbf{Output:}
$\pi_\psi\approx\pie$, 
$r_\theta(s, a)\approx t(s, a;\pie)$ (a solution of the IRL problem in~\lemmaref{lemma1}).
\end{algorithmic}
\end{algorithm}
\end{center}

\Blue{Based on} a solution \Blue{for regularized IRL in the previous section,} we focus on developing an IRL algorithm in this section.
\Blue{Particularly to} recover the reward function $t(s, a;\pie)$ in~\lemmaref{lemma1}, we design an adversarial training objective as follows. 
\Blue{Motivated by} AIRL~\citep{fu2018learning}, we consider the following structured discriminator associated with $\pi, r$ and $t$ in~\lemmaref{lemma1}:
\begin{align*}
    D_{r, \pi}(s, a)
    =
    \sigma(r(s, a)-t(s, a;\pi)), 
    \sigma(z)=\frac{1}{1+e^{-z}}, z\in\R.
\end{align*}
Note that we can recover the discriminator of AIRL in~\eqref{eq:airl1} when
$t(s, a)=\log\pi(a|s)$ ($\phi(x)=\log x$ and $\lambda=1$).
Then, we consider the following optimization objective of the discriminator which is the same as that of AIRL:
\begin{align}
    \hat{t}(s, a;\pi):=\argmax_{r\in\R^\SA}~
    \E_{(s, a)\sim{d}_{\pie}}
    \left[ 
        \log D_{r, \pi}(s, a)
    \right] 
    +
    \E_{(s, a)\sim{d}_\pi}
    \left[ 
        \log (1-D_{r, \pi}(s, a))
    \right].
    \label{eq:obj_rew}
\end{align}
Since the function $x\mapsto a\log \sigma(x)+b\log(1-\sigma(x))$ attains its maximum at $\sigma(x)=\frac{a}{a+b}$, or equivalently at $x=\log\frac{a}{b}$ ~\citep{goodfellow2014generative,mescheder2017adversarial}, it can be shown that 
\begin{align}
    \hat{t}(s, a;\pi)
    =
    t(s, a;\pi)
    +
    \log\frac{{d}_{\pie}(s, a)}{{d}_\pi(s, a)}.
    \label{eq:maximizer}
\end{align}
When $\pi=\pie$ in \eqref{eq:maximizer}, we have
$\hat{t}(s, a;\pie)=t(s, a;\pie)$ since ${d}_\pi={d}_{\pie}$, which means the maximizer $\hat{t}$ becomes the solution of IRL \Blue{\emph{after}} the agent successfully imitates the expert policy $\pie$.
To do so, we consider the following iterative algorithm.
Assuming that we find out the optimal \Blue{reward approximator} $\hat{t}(s, a;\pi^{(i)})$ in \eqref{eq:maximizer} for the policy $\pi^{(i)}$ of the $i$-th iteration, we get the policy $\pi^{(i+1)}$ by optimizing the following objective with gradient ascent:
\begin{align}
    \Blue{
    \underset{\pi\in\P^\A_\S}{\mathrm{maximize}}~
    }
    \E_{(s, a)\sim d_\pi}
    \left[
        \hat{t}(s, a;\pi^{(i)})
        -
        \Omega(\pi(\cdot|s))
    \right].
    \label{eq:policy_optimization_method1}
\end{align}
\Blue{The above} expectation in \eqref{eq:policy_optimization_method1} can be decomposed into the following two \Blue{terms}
\begin{align}
    \E_{(s, a)\sim d_\pi}
    \left[
        \hat{t}(s, a;\pi^{(i)})
        -
        \Omega(\pi(\cdot|s))
    \right]
    &=
    \E_{(s, a)\sim d_\pi}
    \left[
        t(s, a;\pi^{(i)})
        -
        \Omega(\pi(\cdot|s))
    \right]
    -
        \mathrm{KL}(d_\pi||d_{\pie})\nonumber\\
    &=
    -
    \underbrace{
    \E_{(s, a)\sim d_\pi}
    \left[
        D_\Omega^\A(\pi(\cdot|s)||\pi^{(i)}(\cdot|s))
    \right]
    }_{\textrm{(I)}}
    -
    \underbrace{
        \mathrm{KL}(d_\pi||d_{\pie})
    }_{\textrm{(II)}},
    \label{eq:policy_optimization_method2}
\end{align}
where the second equality follows since \lemmaref{lemma1} \Blue{tells} us that $t(s, a;\pi^{(i)})$ is a reward function that makes 
$\pi^{(i)}$ an optimal policy in the $\Omega$-regularized MDP. 
Minimizing term \textrm{(II)} in \eqref{eq:policy_optimization_method2} 
makes $\pi^{(i+1)}$ close to $\pie$ while minimizing term \textrm{(I)} can be regarded as a conservative policy optimization around the policy $\pi^{(i)}$~\citep{schulman2015trust}.

In practice, we parameterize our reward and policy approximations with neural networks and train them using an off-policy Regularized Actor-Critic (RAC)~\citep{yang2019regularized} as described in \textbf{Algorithm~\ref{algorithm:reg_airl}}.
We evaluate our Regularized Adversarial Inverse Reinforcement Learning (RAIRL) approach across various scenarios, below.


\section{Experiments}
\label{experiments}
\vspace{-0.05in}
We summarize the experimental setup as follows.
In our experiments, we consider $\Omega(\Blue{p})=-\lambda\E_{a\sim\Blue{p}}[\phi(\Blue{p}(a)]$ with the following regularizers from~\citet{yang2019regularized}:
\emph{(1) Shannon entropy}
($\phi(x)=-\log x$),
\emph{(2) Tsallis entropy regularizer}
($\phi(x)=\frac{k}{q-1}(1-x^{q-1})$),
\emph{(3) $\mathrm{exp}$ regularizer}
($\phi(x)=e-e^x$),
\emph{(4) $\mathrm{cos}$ regularizer}
($\phi(x)=\cos(\frac{\pi}{2}x)$),
\emph{(5) $\mathrm{sin}$ regularizer}
($\phi(x)=1 - \sin\frac{\pi}{2}x$). 
\Blue{The above regularizers were chosen since other regularizers have not been empirically validated to the best of our knowledge. We chose those regularizers to make our empirical analysis more tractable.}
In addition, we model the reward approximator \Blue{of RAIRL} as a neural network with either one of the following models: \emph{(1) Non-structured model (NSM)}---a simple feed-forward neural network that outputs real values used in AIRL~\citep{fu2018learning}---and \emph{(2) Density-based model (DBM)}---a model using a neural network for $\pi$ (softmax for discrete controls and multi-variate Gaussian model for continuous controls) of the solution in~\eqref{corollary1} \Blue{(See Appendix~\ref{sec:dbm} for detailed explanation)}. \Blue{For RL algorithm of RAIRL, we implement Regularized Actor Critic (RAC)~\citep{yang2019regularized} on top of SAC implementation of  Rlpyt~\citep{stooke2019rlpyt}.}
Other settings are summarized in Appendix~\ref{experiment_setting}. For all experiments, we use 5 runs and report 95\% confidence interval.

\subsection{Experiment 1: Multi-armed Bandit (Discrete Action)}

We consider a 4-armed bandit environment as shown in Figure~\ref{figure_bandit} (left).
An expert policy $\pie$ is assumed to be either \emph{dense} (with probability 0.1, 0.2, 0.3, 0.4 for $a=0, 1, 2, 3$) or \emph{sparse} (with probability 0, 0, 1/3, 2/3 for  $a=0, 1, 2, 3$).
For those experts, we use RAIRL with actions sampled from $\pie$ and compare learned rewards with the ground truth reward $t(s, a;\pie)$ in~\lemmaref{lemma1}.
When $\pie$ is dense, RAIRL successfully acquires the ground truth rewards irrespective of the reward model choices. 
When sparse $\pie$ is used, however, RAIRL with a non-structured model (RAIRL-NSM) failed to recover the rewards for $a=0, 1$---where $\pie(a)=0$---due to the lack of samples at the end of the imitation. 
On the other hand, RAIRL with a density-based model (RAIRL-DBM) can recover the correct rewards due to the softmax layer which maintains the sum over the outputs equal to 1. Therefore, we argue that using DBM is necessary for correct reward acquisition since a set of demonstration is generally sparse. In the following experiment, we show the choice of reward models indeed affects the performance of rewards. 
\begin{center}
    \begin{figure}[h]
        \centering
        \includegraphics[width=\textwidth]{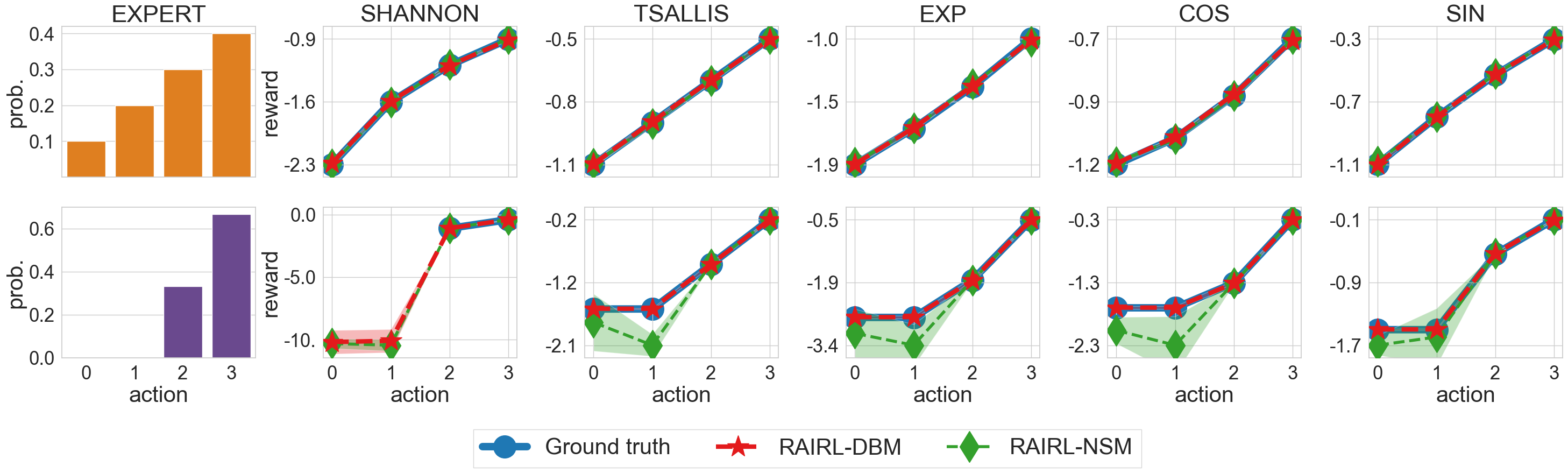}
        \caption{Expert \Blue{policy (\emph{Left}) and reward} learned by RAIRL with different types of policy regularizers \Blue{(\emph{Right}) in Multi-armed Bandit.} Either one of dense (\emph{Top row}) or sparse (\emph{Bottom row}) expert policies $\pie$ is considered.}
        \label{figure_bandit}
    \end{figure}
    \vspace{-25pt}
\end{center}

\subsection{Experiment 2: Bermuda World (Continuous Observation, Discrete Action)}

We consider an environment with a 2-dimensional continuous state space as described in Figure~\ref{fig:figure4}. 
At each episode, the learning agent is initialized uniformly on the $x$-axis between $-5$ and $5$, and there are 8 possible actions---an angle in 
$
\{
    -\pi,
    -\frac{3}{4}\pi,
    ...,
    \frac{1}{2}\pi,
    \frac{3}{4}\pi
\}
$ that determines the direction of movement.
An expert in Bermuda World considers 3 target positions $(-5, 10), (0, 10), (5, 10)$ and behaves stochastically. \Blue{We state how we mathematically define the expert policy $\pie$ in Appendix~\ref{expert_in_bermuda}.}
\Blue{During RAIRL's training (Figure~\ref{fig:figure4}, \emph{Top row}), we use 1000 demonstrations sampled from the expert and periodically measure mean Bregman divergence, i.e., for $D_\Omega^\A(p_1||p_2)=
        \E_{a\sim p_1}
        [f_\phi'(p_2(a))-\phi(p_1(a))]
        -
        \E_{a\sim p_2}
        [f_\phi'(p_2(a))-\phi(p_2(a))]
        $,
\begin{align*}
    \frac{1}{N}\sum_{i=1}^ND_\Omega^\A(\pi(\cdot|s_i)||\pie(\cdot|s_i)).
\end{align*}
Here, the states $s_1, ..., s_N$ comes from 30 evaluation trajectories that are stochastically sampled from the agent's policy $\pi$---which is fixed during evaluation---in a separate evaluation environment. 
During the evaluation of learned reward (Figure~\ref{fig:figure4}, \emph{Bottom row}), we train randomly initialized agents with RAC and rewards acquired from RAIRL's training and check if mean Bregman divergence is properly minimized. We measure \emph{mean Bregman divergence} as was done in RAIRL's training.}

RAIRL-DBM is shown to minimize the target divergence more effectively compared to RAIRL-NSM \Blue{during reward evaluation} although both achieve comparable performances during RAIRL's training.  
Moreover, we substitute $\lambda$ with $1, 5, 10$ and observe that learning with $\lambda$ larger than 1 returns better rewards---only $\lambda=1$ was considered in AIRL~\citep{fu2018learning}.
Note that in all cases, the minimum divergence achieved by RAIRL is comparable with that of behavioral cloning (BC). This is because BC performs sufficiently well when many demonstration is given. We think the divergence of BC may be the near-optimal divergence that can be achieved with our policy neural network model.
\begin{center}
    \vspace{-10pt}
    \begin{figure}[h]
        \centering
        \includegraphics[width=\textwidth]{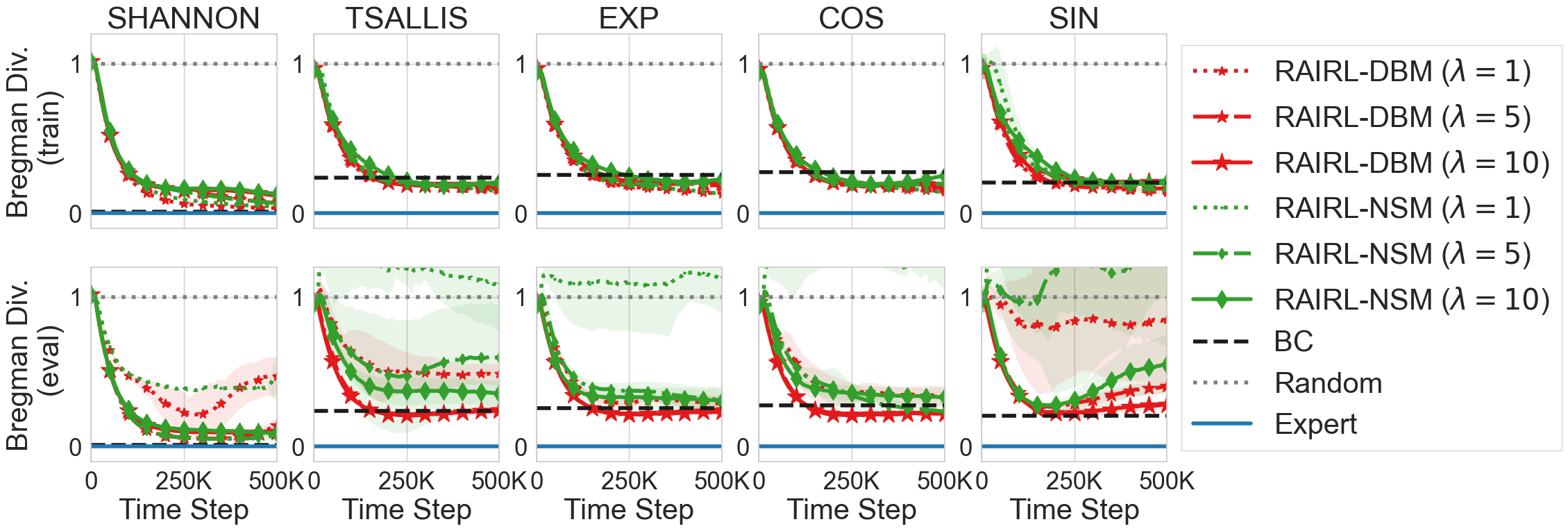}
        \caption{
        \Blue{
        Mean Bregman divergence during training (\emph{Top row}) and the divergence during reward evaluation (\emph{Bottom row}) in Bermuda World. In each column, different policy regularizers and their respective target divergences are considered.
        The results are reported after normalization with the divergence of uniform random policy, and that of behavioral cloning (BC) is reported for comparison.
        }
        }
        \label{fig:figure4}
    \end{figure}
    \vspace{-25pt}
\end{center}


\subsection{Experiment 3: MuJoCo (Continuous Observation and Action)}

We \Blue{validate} RAIRL on MuJoCo continuous control tasks (\emph{Hopper-v2}, \emph{Walker-v2}, \emph{HalfCheetah-v2}, \emph{Ant-v2}) as follows. 
We assume multivariate-Gaussian policies (with diagonal covairance matrices) for both learner's policy $\pi$ and expert policy $\pie$. Instead of $\mathrm{tanh}$-squashed policy in Soft-Actor Critic~\citep{haarnoja2018soft}, we use hyperbolized environments---where $\mathrm{tanh}$ is regarded as a part of the environment---with additional engineering on the policy networks (See Appendix~\ref{appendix_mujoco} for details).
We use 100 demonstrations \Blue{stochastically sampled} from $\pie$ to \Blue{validate} RAIRL.
In MuJoCo experiments, we focus on a set of Tsallis entropy regularizer (\Blue{$\Omega=-\mathcal{T}_q^1$}) with $q=1, 1.5, 2$---where Tsallis entropy becomes Shannon entropy for $q=1$.
We then exploit the \Blue{tractable} quantities for multi-variate Gaussian distributions in Section~\ref{IRLTSALLIS} to stabilize RAIRL and \Blue{check} its performance in terms of \Blue{mean Bregman divergence similar to the previous experiment. Note that since both $\pi$ and $\pie$ are multi-variate Gaussian and can be evaluated, we can evaluate the individual Bregman divergence $D_\Omega^\A(\pi(\cdot|s)||\pie(\cdot|s))$ on $s$ by using the derivation in Appendix~\ref{bregman_tsallis}.}

\Blue{The performances during RAIRL's training} are \Blue{described as follows.}
We \Blue{report} $\pi$ with both an episodic score (Figure~\ref{fig:mujoco_score}) and \Blue{mean Bregman divergences with respect to three types of Tsallis entropies} (Figure~\ref{fig:mujoco_bregman}) \Blue{$\Omega=-\mathcal{T}_{q'}^1$} with $q'=1, 1.5, 2$.
Note that the objective of RAIRL with $\Omega=-\mathcal{T}_{q}^\Blue{1}$ is to minimize the corresponding mean Bregman divergence with $q'=q$.
In Figure~\ref{fig:mujoco_score}, both RAIRL-DBM and RAIRL-NSM are shown to achieve the expert performance, irrespective of $q$, in \emph{Hopper-v2}, \emph{Walker-v2}, and \emph{HalfCheetah-v2}. In contrast, RAIRL in \emph{Ant-v2} fails to achieve the expert's performance within 2,000,000 steps and RAIRL-NSM highly outperforms RAIRL-DBM in our setting. 
Although the \Blue{episodic} scores are comparable for all methods in \emph{Hopper-v2}, 
\emph{Walker-v2}, and \emph{HalfCheetah-v2}, 
respective divergences are shown to be highly different from one another as shown in Figure~\ref{fig:mujoco_bregman}.
RAIRL with $q=2$ in most cases achieves the minimum mean Bregman divergence (for all three divergences with  $q'=1, 1.5, 2$), whereas RAIRL with $q=1$---which corresponds to AIRL~\citep{fu2018learning}---achieves the maximum divergence in most cases.
This result is in alignment with our intuition from  Section~\ref{IRLTSALLIS}; as $q$ increases, minimizing the Bregman divergence requires much tighter matching between $\pi$ and $\pie$.
\Blue{Unfortunately, while evaluating the acquired reward---RAC with randomly initialized agent and acquired reward---the target divergence is not properly decreased in continuous controls.} 
We believe this is because $\pi$ is a probability density function in continuous controls and causes large variance during training, while $\pi$ is a mass function and is well-bounded in discrete control problems.

\begin{center}
    \vspace{-10pt}
    \begin{figure}[h]
        \centering
        \includegraphics[width=\textwidth]{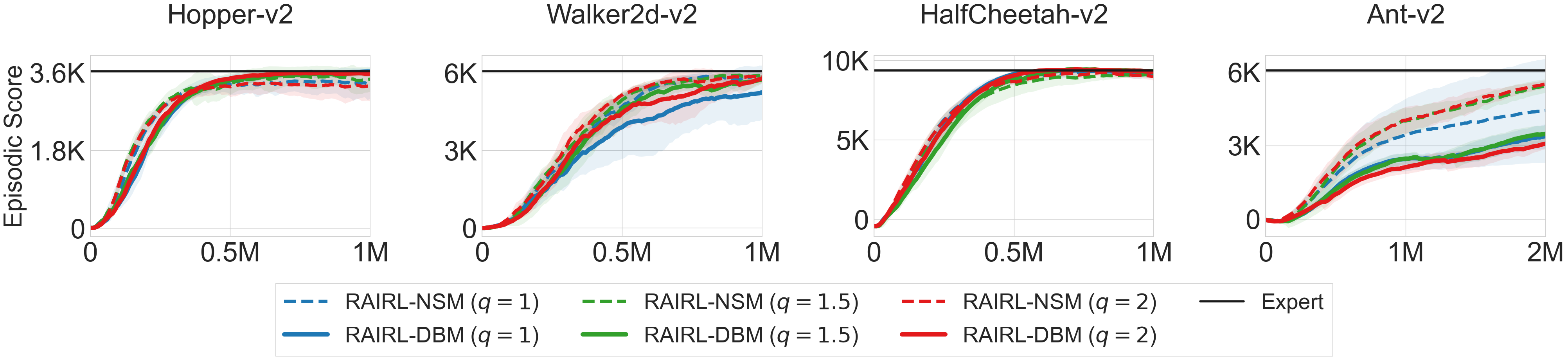}
    \caption{
    \Blue{
    Averaged episodic score of RAIRL's training in MuJoCo environments. RAIRL with $\mathcal{T}_{q}^1$ regularizer with $q=1, 1.5, 2$ is considered.
    }
    }
    \label{fig:mujoco_score}
    \end{figure}
    \vspace{-25pt}
\end{center}

\begin{center}
    \begin{figure}[h]
        \centering
        \includegraphics[width=\textwidth]{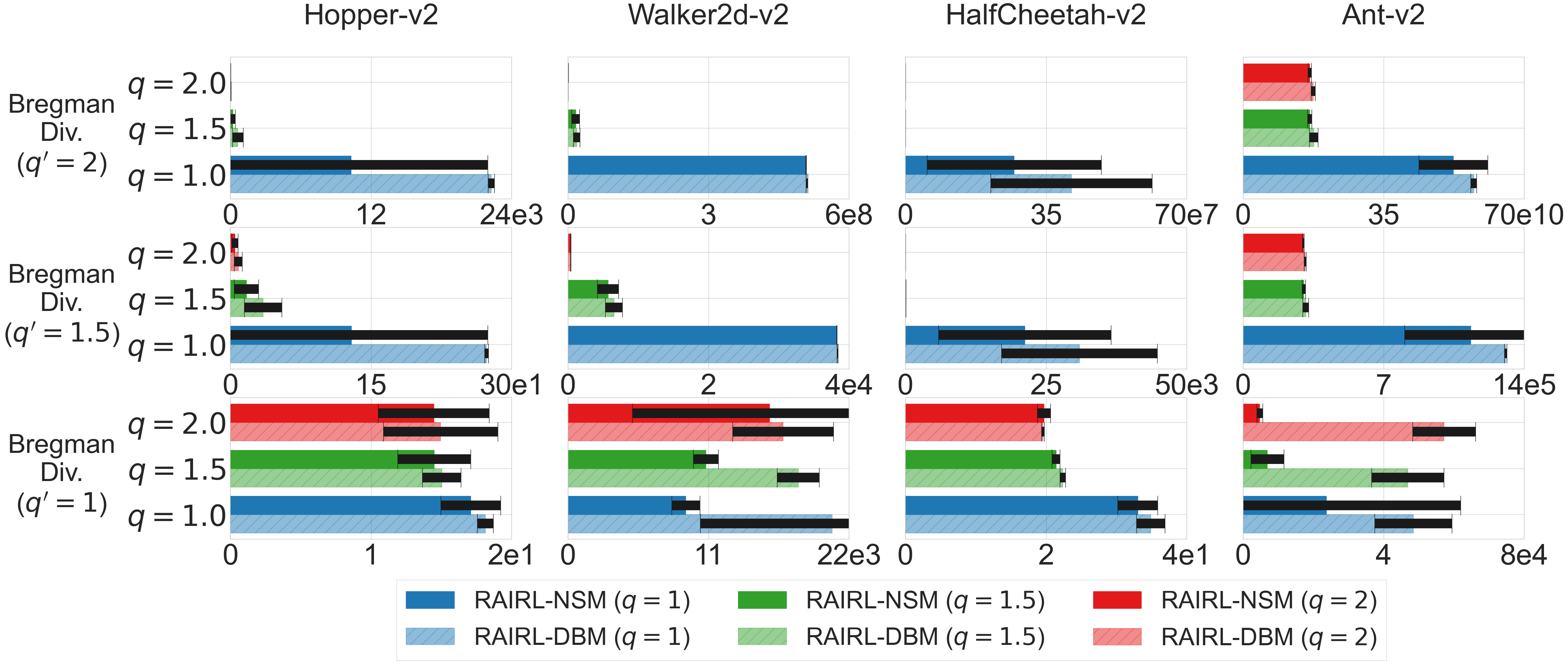}
    \caption{
    \Blue{
    Bregman divergences with Tsallis entropy $\mathcal{T}_{q'}^1$ with $q'=1, 1.5, 2$ during RAIRL's training in MuJoCo environments. We consider RAIRL with Tsallis entropy regularizer $\mathcal{T}_{q}^1$ with $q=1, 1.5, 2$.
    }
    }
    \label{fig:mujoco_bregman}
    \end{figure}
    \vspace{-30pt}
\end{center}

        

\section{Discussion and Future Works}

We consider the problem of IRL in regularized MDPs~\citep{geist2019theory}, assuming a  class of \Blue{strongly} convex policy regularizers.
We theoretically derive its solution (a set of reward functions) and show that learning with these rewards is equivalent to a specific instance of imitation learning---i.e., one that minimizes the Bregman divergence associated with
policy regularizers.
We propose RAIRL---a practical sampled-based IRL algorithm in regularized MDPs---and evaluate its applicability on policy imitation (for discrete and continuous controls) and reward acquisition (for discrete control).

Finally, recent advances in imitation learning and IRL are built from the perspective of regarding imitation learning as statistical divergence minimization problems~\citep{ke2019imitation,ghasemipour2019divergence}. Although Bregman divergence is known to cover various divergences, it does not include some divergence families such as $f$-divergence~\citep{csiszar63,amari2009alpha}. Therefore, we believe that considering RL with policy regularization different from \citet{geist2019theory} and its inverse problem is a possible way of finding the links between imitation learning and various statistical distances.

\bibliography{iclr2021_conference}
\bibliographystyle{iclr2021_conference}

\newpage
\appendix

\Blue{
\section{Bellman Operators, Value Functions in regularized MDPs}
\label{regularizedrl}

Let a policy regularizer $\Omega:\P^\A\rightarrow\R$ be strongly convex, and define the convex conjugate of $\Omega$ is $\Omega^*:\R^\A\rightarrow\R$ as
\begin{align}
    \Omega^*(Q(s, \cdot))
    &=
    \max_{\pi(\cdot|s)\in\P^\A}
    \inner{\pi(\cdot|s)}{Q(s, \cdot)}{\A}
    - \Omega(\pi(\cdot|s)), Q\in\R^\SA, s\in\S.
    \label{eq:geist1}
\end{align}
Then, Bellman operators, equations and value functions in regularied MDPs are defined as follows.
\begin{definition}[Regularized Bellman Operators]
\label{definition1}
For $V\in\R^\S$, let us define $Q(s, a)=r(s, a)+\gamma\E_{s'\sim P(\cdot|s, a)}[V(s')]$. 
The regularized Bellman evaluation operator is defined as
\begin{align*}
    [T^{\pi}V](s)
    :=
    \inner{\pi(\cdot|s)}{Q(s, \cdot)}{\A}-\Omega(\pi(\cdot|s)),
    s\in\S,
\end{align*}
and $T^{\pi}V=V$ is called the regularized Bellman equation. Also, the regularized Bellman optimality operator is defined as
\begin{align*}
    [T^{*}V](s)
    :=
    \max_{\pi(\cdot|s)\in\P^\A}
    [T^{\pi}V](s)
    =
    \Omega^*(Q(s, \cdot)), 
    s\in\S,
\end{align*}
and $T^{*}V=V$ is called the regularized Bellman optimality equation.
\end{definition}

\begin{definition}[Regularized value functions]
The unique fixed point $V^\pi$ of the operator $T^\pi$ is called the state value function, and $Q^\pi(s, a)=r(s, a)+\gamma\E_{s'\sim P(\cdot|s, a)}[V^\pi(s')]$ is called the state-action value function. Also, the unique fixed point $V^*$ of the operator $T^*$ is called the optimal state value function, and $Q^*(s, a)=r(s, a)+\gamma\E_{s'\sim P(\cdot|s, a)}[V^*(s')]$ is called the optimal state-action value function. 
\end{definition}
It should be noted that \textbf{Proposition 1} in \cite{geist2019theory} tells us $\nabla\Omega^*(Q(s, \cdot))$ is a policy that \emph{uniquely} maximizes \eqref{eq:geist1}.
For example, when $\Omega(\pi(\cdot|s))=\sum_{a\sim \pi(\cdot|s)}\log \pi(a|s)$ (negative Shannon entropy), 
$\nabla\Omega^*(Q(s, \cdot))$ is a softmax policy, i.e.,
$
    \nabla\Omega^*(Q(s, \cdot))
    =
    \frac{\exp(Q(s, \cdot))}{\sum_{a'\in\A}\exp(Q(s, a'))}
$.
Due to this property, the optimal policy $\pi^*$ of a regularized MDP is uniquely found for the optimal state-action value function $Q^*$ which is also uniquely defined as the fixed point of $T^*$.
}
\section{Proof of \lemmaref{lemma1}}
\label{proofoflemma1}

Let us define $\pi^s=\pi(\cdot|s)$. 
For $r(s, a)=t(s, a;\pie)$, the RL objective \eqref{RL_OBJECTIVE} satisfies
\begin{align}
    &\E_\pi
    \left[
        \sum_{i=0}^\infty
        \gamma^i 
        \left\{
            r(s_i, a_i) - \Omega(\pi^{s_{i}})
        \right\}
    \right]
    \overset{(\textrm{i})}{=}
    \E_\pi
    \left[
        \sum_{i=0}^\infty
        \gamma^i 
        \left\{
            t(s, a;\pie) - \Omega(\pi^{s_{i}})
        \right\}
    \right]
    \nonumber
    \\
    &\overset{(\textrm{ii})}{=}
    \E_\pi
    \left[
        \sum_{i=0}^\infty
        \gamma^i 
        \left\{
            \Omega'(s_i, a_i;\pie)
            -
            \E_{a\sim\pie^{s_i}}\Omega'(s_i, a;\pie)+\Omega(\pie^{s_{i}})
            -
            \Omega(\pi^{s_{i}})
        \right\}
    \right]
    \nonumber
    \\
    &\overset{(\textrm{iii})}{=}
    \E_\pi
    \left[
        \sum_{i=0}^\infty
        \gamma^i 
        \left\{
            \E_{a\sim\pi^{s_{i}}}\Omega'(s_i, a;\pie)
            -
            \E_{a\sim\pie^{s_{i}}}\Omega'(s_i, a;\pie)
            +
            \Omega(\pie^{s_{i}})
            -
            \Omega(\pi^{s_{i}})
        \right\}
    \right]
    \nonumber
    \\
    &\overset{(\textrm{iv})}{=}
    \E_\pi
    \left[
        \sum_{i=0}^\infty
        \gamma^i 
        \left\{
            \inner{\nabla\Omega(\pie^{s_{i}})}{\pi^{s_{i}}}{\A}
            -
            \inner{\nabla\Omega(\pie^{s_{i}})}{\pie^{s_{i}}}{\A}
            +
            \Omega(\pie^{s_{i}})
            -
            \Omega(\pi^{s_{i}})
        \right\}
    \right]
    \nonumber
    \\
    &=
    \E_\pi
    \left[
        \sum_{i=0}^\infty
        \gamma^i 
        \left\{
            \Omega(\pie^{s_{i}})
            -
            \Omega(\pi^{s_{i}})
            +
            \inner{\nabla\Omega(\pie^{s_{i}})}{\pi^{s_{i}}}{\A}
            -
            \inner{\nabla\Omega(\pie^{s_{i}})}{\pie^{s_{i}}}{\A}
        \right\}
    \right]
    \nonumber
    \\
    &=
    -\E_\pi
    \left[
        \sum_{i=0}^\infty
        \gamma^i 
        \left\{
            \Omega(\pi^{s_{i}})
            -
            \Omega(\pie^{s_{i}})
            -
            \inner{\nabla\Omega(\pie^{s_{i}})}{\pi^{s_{i}}}{\A}
            +
            \inner{\nabla\Omega(\pie^{s_{i}})}{\pie^{s_{i}}}{\A}
        \right\}
    \right]
    \nonumber
    \\
    &=
    -\E_\pi
    \left[
        \sum_{i=0}^\infty
        \gamma^i 
        \left\{
            \Omega(\pi^{s_{i}})
            -
            \Omega(\pie^{s_{i}})
            -
            \inner{\nabla\Omega(\pie^{s_{i}})}{\pi^{s_{i}}-\pie^{s_{i}}}{\A}
        \right\}
    \right]
    \overset{(\textrm{v})}{=}
    -\E_\pi
    \left[
        \sum_{i=0}^\infty
        \gamma^i 
        D_\Omega^{\A}(\pi^{s_i}||\pie^{s_{i}}))
    \right],
    \label{eq:proofoflemma1:1}
\end{align}
where 
\textrm{(i)} follows from the assumption $r(s, a)=t(s, a;\pie)$ in \lemmaref{lemma1}, 
\textrm{(ii)} follows from the definition of $t(s, a;\pi)$ in \eqref{TARGET_REWARD},
\textrm{(iii)} follows since taking the inner expectation first does not change the overall expectation,
\textrm{(iv)} follows 
from the definition of $\Omega'$ in \lemmaref{lemma1} 
and 
$
\sum_{a\in\A}\Blue{p}(a)[\nabla\Omega(\Blue{p})](a)=\inner{\nabla\Omega(\Blue{p})}{\Blue{p}}{\A}
$,
and
\textrm{(v)} follows from the definition of Bregman divergence, i.e.,
$D_\Omega^{\A}(\Blue{p_1}||\Blue{p_2})=\Omega(\Blue{p_1})-\Omega(\Blue{p_2})-\inner{\nabla\Omega(\Blue{p_2})}{\Blue{p_1}-\Blue{p_2}}{\A}$.
Due to the non-negativity of $D_\Omega^{\A}$, \eqref{eq:proofoflemma1:1} is less than or equal to zero which can be achieved when $\pi=\pie$.

\section{Proof of \corollaryref{corollary1}}
\label{proofofcorollary1}
Let $a\in\{1, ..., |\A|\}$ and $\pi_a=\pi(a)$ for simplicity.
For 
\begin{align*}
\Omega(\pi)
=
-\lambda\E_{a\sim\pi}\phi(\pi_a)
=
-\lambda\sum_{a\in\A}\pi_a\phi(\pi_a)
=
-\lambda\sum_{a\in\A}f_\phi(\pi_a)
\end{align*}
with $f_\phi(x)=x\phi(x)$,
we have
\begin{align*}
    \nabla\Omega(\pi)
    =
    -\lambda
    \nabla_{\pi_1,...,\pi_{|\A|}}{\sum_{a\in\A}f_\phi(\pi_a)}
    =
    -\lambda[f_\phi'(\pi_1), ..., f_\phi'(\pi_{|\A|})]^T
\end{align*}
for $f_\phi'(x)=\frac{\partial}{\partial x}(x\phi(x))$. Therefore, for $\pi^s=\pi(\cdot|s)$ we have
\begin{align*}
    t(s, a;\pi)
    &=
    [\nabla\Omega(\pi^s)](a)
    -
    \inner{\nabla\Omega(\pi^s)}{\pi^s}{\A}
    +
    \Omega(\pi^s)
    \\
    &=
    -\lambda f_\phi'(\pi^s_a)
    -\left(\sum_{a'\in\A}\pi^s_{a'}(-\lambda f_\phi'(\pi^s_{a'}))\right)
    +\left(-\lambda\sum_{{a'}\in\A}\pi^s_{a'}\phi(\pi^s_{a'})\right)
    \\
    &=
    -\lambda 
    \left\{
        f_\phi'(\pi^s_a)
        -
        \left(
            \sum_{a'\in\A}\pi^s_{a'}f_\phi'(\pi^s_{a'})
        \right)
        +
        \left(
            \sum_{a'\in\A}\pi^s_{a'}\phi(\pi^s_{a'})
        \right)
    \right\}
    \\
    &=
    -\lambda 
    \left\{
        f_\phi'(\pi^s_a)
        -
        \sum_{a'\in\A}\pi^s_{a'}
        \left(
            f_\phi'(\pi^s_{a'})
            -
            \phi(\pi^s_{a'})
        \right)
    \right\}
    \\
    &=
    -\lambda 
    \left\{
        f_\phi'(\pi^s_a)
        -
        \E_{a'\sim\pi^s}
        \left[
            f_\phi'(\pi^s_{a'})
            -
            \phi(\pi^s_{a'})
        \right]
    \right\}.
\end{align*}

\section{Proof of Optimal Rewards on Continuous Controls}\label{proofforcontinuouscontrol}

Note that for two continuous distributions $\mathbb{P}_1$ and $\mathbb{P}_2$ having probability density functions $p_1(x)$ and $p_2(x)$, respectively, the Bregman divergence can be defined as~\citep{guo2017ambiguity,jones1990general}
\begin{align}
D_\omega^{\X}(\mathbb{P}_1||\mathbb{P}_2)
:=
\int_\X
\left\{
    \omega(p_1(x))
    -
    \omega(p_2(x))
    -
    \omega'(p_2(x))
    (p_1(x)-p_2(x))
\right\}
dx,
\label{EQD}
\end{align}
where $\omega'(x):=\frac{\partial}{\partial x} \omega(x)$ and the divergence is measure point-wisely on $x\in\X$.
Let us assume 
\begin{align}
    \Omega(\pi)
    =
    \int_\A\omega(\pi(a))da
    \label{EQOMEGA}
\end{align}
for a probability density function $\pi$ on $\A$ and define 
\begin{align}
    t(s, a;\pi)
    :=
    \omega'(\pi^s(a))
    -
    \int_{\A}
    \left[
        \pi^s(a')\omega'(\pi^s(a'))
        -
        \omega(\pi^s(a'))
    \right]
    da'.
    \label{EQT}
\end{align}
for $\pi^s=\pi(\cdot|s)$.
For $r(s, a)=t(s, a;\pie)$, the RL objective in~\eqref{RL_OBJECTIVE} satisfies
\begin{align}
    &\E_\pi
    \left[
        \sum_{i=0}^\infty
        \gamma^i 
        \left\{
            r(s_i, a_i) - \Omega(\pi^{s_{i}})
        \right\}
    \right]
    =
    \E_\pi
    \left[
        \sum_{i=0}^\infty
        \gamma^i 
        \left\{
            t(s_i, a_i;\pie) 
            -
            \Omega(\pi^{s_{i}})
        \right\}
    \right]
    \nonumber
    \\
    &=
    \E_\pi
    \left[
        \sum_{i=0}^\infty
        \gamma^i 
        \left\{
            \omega'(\pie^{s_i}(a_i))
            -
            \int_\A
            \left[
                \pie^{s_i}(a)\omega'(\pie^{s_i}(a))
                -
                \omega(\pie^{s_i}(a))
            \right]
            da
            -
            \int_\A
            \omega(\pi^{s_{i}}(a))
            da
        \right\}
    \right]
    \nonumber
    \\
    &=
    \E_\pi
    \left[
        \sum_{i=0}^\infty
        \gamma^i 
        \left\{
            \int_\A
            \pi^{s_i}(a)
            \omega'(\pie^{s_i}(a))
            da
            -
            \int_\A
            \left[
                \pie^{s_i}(a)\omega'(\pie^{s_i}(a))
                -
                \omega(\pie^{s_i}(a))
                +
                \omega(\pi^{s_{i}}(a))
            \right]
            da
        \right\}
    \right]
    \nonumber
    \\
    &=
    \E_\pi
    \left[
        \sum_{i=0}^\infty
        \gamma^i 
        \int_\A
        \left\{
            \pi^{s_i}(a)
            \omega'(\pie^{s_i}(a))
            -
            \left[
                \pie^{s_i}(a)\omega'(\pie^{s_i}(a))
                -
                \omega(\pie^{s_i}(a))
                +
                \omega(\pi^{s_{i}}(a))
            \right]
        \right\}
        da
    \right]
    \nonumber
    \\
    &=
    \E_\pi
    \left[
        \sum_{i=0}^\infty
        \gamma^i 
        \int_\A
        \left\{
            \omega(\pie^{s_i}(a))
            -
            \omega(\pi^{s_i}(a))
            +
            \pi^{s_i}(a)
            \omega'(\pie^{s_i}(a))
            -
            \pie^{s_i}(a)\omega'(\pie^{s_i}(a))
        \right\}
        da
    \right]
    \nonumber
    \\
    &=
    -
    \E_\pi
    \left[
        \sum_{i=0}^\infty
        \gamma^i 
        \int_\A
        \left\{
            \omega(\pi^{s_i}(a))
            -
            \omega(\pie^{s_i}(a))
            -
            \pi^{s_i}(a)
            \omega'(\pie^{s_i}(a))
            +
            \pie^{s_i}(a)\omega'(\pie^{s_i}(a))
        \right\}
        da
    \right]
    \nonumber
    \\
    &=
    -
    \E_\pi
    \left[
        \sum_{i=0}^\infty
        \gamma^i 
        \int_\A
        \left\{
            \omega(\pi^{s_i}(a))
            -
            \omega(\pie^{s_i}(a))
            -
            \omega'(\pie^{s_i}(a))
            \left(
                \pi^{s_i}(a)
                -
                \pie^{s_i}(a)
            \right)
        \right\}
        da
    \right]
    \nonumber
    \\
    &=
    -\E_\pi
    \left[
        \sum_{i=0}^\infty
        \gamma^i 
        D_\omega^{\A}(\pi^{s_i}||\pie^{s_{i}}))
    \right],
    \label{EQLAST}
\end{align}
where 
\textrm{(i)} follows from $r(s, a)=t(s, a;\pie)$,
\textrm{(ii)} follows from \eqref{EQOMEGA} and \eqref{EQT},
and
\textrm{(iii)} follows from the definition of Bregman divergence in~\eqref{EQD}.
Due to the non-negativity of $D_\omega$, \eqref{EQLAST} is less than or equal to zero which can be achieved when $\pi=\pie$.
Also, $\pi=\pie$ is a unique solution since \eqref{RL_OBJECTIVE} has a unique solution for arbitrary reward functions. 

\Blue{
\section{Proof of \lemmaref{lemma2}}
\label{proofoflemma2}
Since \lemmaref{lemma2} was mentioned but not proved in \cite{geist2019theory}, we remain the rigorous proof for \lemmaref{lemma2} in this subsection. Note that we follow the proof idea in \cite{ng1999policy}.

First, let us assume an MDP $\M^r$ with a reward $r$ and corresponding optimal policy $\pi^*$. From \textbf{Definition~\ref{definition1}}, the optimal state-value function $V^{*,r}$ and its corresponding state-action value function $Q^{*,r}(s, a):=r(s, a)+\gamma\E_{s'\sim P(\cdot|s, a)}[V^{*,r}(s')]$ should satisfies the regularized Bellman optimality equation
\begin{align}
    V^{*,r}(s)
    &=T^{*,r}V^{*,r}(s)
    =
    \Omega^*(Q^{*,r}(s, a))
    \nonumber
    \\
    &=
    \max_{\pi(\cdot|s)\in\P^\A}
    \inner{\pi(\cdot|s)}{Q^{*,r}(s, a)}{\A}-\Omega(\pi(\cdot|s))
    \nonumber
    \\
    &=
    \max_{\pi(\cdot|s)\in\P^\A}
    \inner{\pi(\cdot|s)}{r(s, a)+\gamma\E_{s'\sim P(\cdot|s, a)}[V^{*,r}(s')]}{\A}-\Omega(\pi(\cdot|s)),
    \label{eq:optimalityforr}
\end{align} 
where we explicitize the dependencies on $r$. Also, the optimal policy $\pi^*$ is a unique maximizer for the above maximization.

Now, let us consider the shaped rewards $r(s, a)+\Phi(s')-\Phi(s)$ and $r(s, a)+\E_{s'\sim P(\cdot|s, a)}\Phi(s')-\Phi(s)$. Please note that for both rewards, the expectation over $s'$ for given $s, a$ is equal to
\begin{align*}
    \tilde{r}(s, a)=r(s, a)+\E_{s'\sim P(\cdot|s, a)}\Phi(s')-\Phi(s),
\end{align*}
and thus, it is sufficient to only consider the optimality for $\tilde{r}$. By subtracting $\Phi(s)$ from the regularized optimality Bellman equation for $r$ in \eqref{eq:optimalityforr}, we have
\begin{align}
    &V^{*,r}(s)-\Phi(s)
    \nonumber
    \\
    &=
    \max_{\pi(\cdot|s)\in\P^\A}
    \inner{\pi(\cdot|s)}{
    r(s, a)
    +\gamma\E_{s'\sim P(\cdot|s, a)}\Phi(s')
    -
    \Phi(s)
    +\gamma\E_{s'\sim P(\cdot|s, a)}[V^{*,r}(s')-\Phi(s')]}{\A}
    \nonumber
    \\
    &
    ~~~~~~~~~~~~~~~~~~~~-\Omega(\pi(\cdot|s))
    \nonumber
    \\
    &=
    \max_{\pi(\cdot|s)\in\P^\A}
    \inner{\pi(\cdot|s)}{
    \tilde{r}(s, a)
    +\gamma\E_{s'\sim P(\cdot|s, a)}[V^{*,r}(s')-\Phi(s')]}{\A} -\Omega(\pi(\cdot|s))
    \nonumber
    \\
    &=
    [T^{*,\tilde{r}}(V^{*,r}-\Phi)](s).
    \nonumber
\end{align}
That is, $V^{*,r}-\Phi$ is the fixed point of the regularized Bellman optimality operator $T^{*,\tilde{r}}$ associated with the shaped reward $\tilde{r}$. Also, a maximizer for the above maximization is $\pi^*$ since subtracting $\Phi(s)$ from \eqref{eq:optimalityforr} does not change the unique maximizer $\pi^*$. 
}

\Blue
{
\section{Comparison between our solution and existing solution in \cite{geist2019theory}}
\subsection{Problems of Existing Solutions in \cite{geist2019theory}}
\label{intractable_form_solution}

In \textbf{Proposition 5} of \cite{geist2019theory}, a solution of regularized IRL is given, and we rewrite the relevant theorem in this subsection. 
Let us consider a regularized IRL problem, where $\pie\in\P_\S^\A$ is an expert policy. 
Assuming that both the model (dynamics, discount factor and regularizer) and the expert policy are known, 
\cite{geist2019theory} proposed a solution of regularized IRL as follows:
\begin{lemma}[A solution of regularized IRL from \cite{geist2019theory}]
\label{lemma4}
Let $\Qe\in\R^\SA$ be a function such that
$
    \pi_E(\cdot|s)
    =
    \nabla\Omega^*(\Qe(s, \cdot)), s\in\S
$.
Also, define
\begin{align*}
    \re(s, a)
    &:=
    \Qe(s, a)-\gamma\E_{s'\sim P(\cdot|s, a)}[\Omega^*(\Qe(s', \cdot))]\\
    &=
    \Qe(s, a)-\gamma\E_{s'\sim P(\cdot|s, a)}[\inner{\pie(\cdot|s')}{\Qe(s', \cdot)}{\A}-\Omega(\pie(\cdot|s'))].
\end{align*}
Then in the $\Omega$-regularized MDP with the reward $\re$, $\pie$ is an optimal policy.
\end{lemma}
\begin{proof}
Although the brief idea of the proof is given in \cite{geist2019theory}, we rewrite the proof in a more rigorous way as follows.
Let us define $\Ve(s)=\inner{\pie(\cdot|s)}{\Qe(s, \cdot)}{\A}-\Omega(\pie(\cdot|s))$. 
Then, $\re(s, a)=\Qe(s, a) - \gamma\E_{s'\sim P(\cdot|s, a)}[\Ve(s')]$, and thus, $\Qe(s, a)=\re(s, a)+\gamma\E_{s'\sim P(\cdot|s, a)}[\Ve(s')]$. By using this and regularized Bellman optimality operator, we have
\begin{align*}
    [T^*\Ve](s)
    &\overset{(\textrm{i})}{=}
    \Omega^*(\Qe(s, \cdot))
    \overset{(\textrm{ii})}{=}
    \max_{\pi(\cdot|s)\in\P^\A}
    \inner{\pi(\cdot|s)}{\Qe(s, \cdot)}{\A}-\Omega(\pi(\cdot|s))\\
    &\overset{(\textrm{iii})}{=}
    \inner{\pie(\cdot|s)}{\Qe(s, \cdot)}{\A}-\Omega(\pie(\cdot|s))
    =
    \Ve(s),
\end{align*}
where (i) and (ii) follow from \textbf{Definition~\ref{definition1}}, and (iii) follows since $\pie$ is a unique maximizer. 
Thus, $\Ve$ is the fixed point of $T^*$, and 
$
    \pi_E(\cdot|s)
    =
    \nabla\Omega^*(\Qe(s, \cdot)), s\in\S,
$
becomes an optimal policy. 
\end{proof}
For example, when negative Shannon entropy is used as a regularizer, we can get $\re(s, a)=\log\pie(a|s)$ by setting $\Qe(s, a)=\log\pie(a|s)$.
However, a solution proposed in \lemmaref{lemma4} has two crucial issues:
\newline
\textbf{Issue 1.} It requires the knowledge on the model dynamics, which is generally intractable.
\newline
\textbf{Issue 2.} We need to figure out $\Qe$ that satisfies $\pie(\cdot|s)=\nabla\Omega^*(\Qe(s, \cdot))$, which comes from the relationship between optimal value function and optimal policy~\citep{geist2019theory}.

In the following subsection, we show how our solution in \lemmaref{lemma1} is related to the solution from \cite{geist2019theory} in \lemmaref{lemma4}.

\subsection{From the solution of \cite{geist2019theory} to the tractable solution}
Let us consider the expert policy $\pie$ and quantities $\Qe$ and $\re$ satisfying the conditions in \lemmaref{lemma4}. From \lemmaref{lemma2}, a regularized MDP with the shaped reward
\begin{align*}
    \tildere(s, a):=\re(s, a)+\gamma\E_{s'\sim P(\cdot|s, a)}\Phi(s') - \Phi(s)
\end{align*}
for $\Phi\in\R^\S$ has its optimal policy as $\pie$.
Since $\Phi$ can be arbitrarily chosen, let us assume $\Phi(s)=\Omega^*(\Qe(s, \cdot))$. Then, we have
\begin{align}
    \tildere(s, a)
    &=
    \re(s, a)+\gamma\E_{s'\sim P(\cdot|s, a)}\Phi(s') - \Phi(s)
    \nonumber
    \\
    &=
    \left\{
        \Qe(s, a)-\gamma\E_{s'\sim P(\cdot|s, a)}[\Omega^*(\Qe(s', \cdot))]
    \right\}
    +
        \E_{s'\sim P(\cdot|s, a)}\Omega^*(\Qe(s', \cdot)) - \Omega^*(\Qe(s, \cdot))
    \nonumber
    \\
    &=
    \Qe(s, a)
    -
    \Omega^*(\Qe(s, \cdot)).
    \label{eq:newsolution1}
\end{align}
Note that the reward in $\eqref{eq:newsolution1}$ does not require the knowledge on the model dynamics, which addresses \textbf{Issue 1} in Appendix~\ref{intractable_form_solution}.
Also, by using $\Ve(s)=\Omega^*(\Qe(s, \cdot))$ in the proof of \lemmaref{lemma4}, the reward in $\eqref{eq:newsolution1}$ can be written as 
\begin{align*}
    \tildere(s, a)
    =
    \Qe(s, a) - \Ve(s),
\end{align*}
which means $\tildere(s, a)$ is \emph{an advantage function} for the optimal policy $\pie$ in the $\Omega$-regularized MDP. 

However, we still have \textbf{Issue 2} in Appendix~\ref{intractable_form_solution} since $\Omega^*$ in $\eqref{eq:newsolution1}$ is generally intractable~\citep{geist2019theory}, which prevents us to find out appropriate $\Qe(s, a)$. Interestingly, we show that 
for all $s\in S$ and $\Qe(s, \cdot)=\nabla\Omega(\pie(\cdot|s))$,
\begin{align*}
    \nabla\Omega^*(\Qe(s, \cdot))
    &=
    \argmax_{\pi(\cdot|s)\in\P^\A}
    \inner{\pi(\cdot|s)}{\Qe(s, \cdot)}{\A}-\Omega(\pi(\cdot|s))\\
    &=
    \argmax_{\pi(\cdot|s)\in\P^\A}
    \inner{\pi(\cdot|s)}{\nabla\Omega(\pie(\cdot|s))}{\A}-\Omega(\pi(\cdot|s))\\
    &=
    \argmin_{\pi(\cdot|s)\in\P^\A}
    \Omega(\pi(\cdot|s))
    -
    \inner{\nabla\Omega(\pie(\cdot|s))}{\pi(\cdot|s)}{\A}\\
    &=
    \argmin_{\pi(\cdot|s)\in\P^\A}
    \Omega(\pi(\cdot|s))-\Omega(\pie(\cdot|s))
    -
    \inner{\nabla\Omega(\pie(\cdot|s))}{\pi(\cdot|s)-\pie(\cdot|s)}{\A}\\
    &=
    \argmin_{\pi(\cdot|s)\in\P^\A}
    D_\Omega^\A(\pi(\cdot|s)||\pie(\cdot|s))=\pie(\cdot|s),
\end{align*}
where the last equality holds since the Bregman divergence $D_\Omega^\A(\pi(\cdot|s)||\pie(\cdot|s))$ is greater than or equal to zero and its lowest value is achieved when $\pi(\cdot|s)=\pie(\cdot|s)$.
This means that when $\Qe(s, \cdot)=\nabla\Omega(\pie(\cdot|s))$ is used, the condition $\pie(\cdot|s)=\nabla\Omega^*(\Qe(s, \cdot))$ in \lemmaref{lemma4} is satisfied \emph{without} knowing the tractable form of $\Omega^*$ or $\nabla\Omega^*$, and thus, \textbf{Issue 2} in Appendix~\ref{intractable_form_solution} is addressed---while we require the knowledge on the gradient $\nabla\Omega$ of the policy regularizer $\Omega$, which is practically more tractable.  
Finally, by substituting $\Qe(s, \cdot)=\Omega'(s, \cdot;\pie)$ for $\Omega'(s, \cdot;\pi):=\nabla\Omega(\pi(\cdot|s)), s\in\S,$ to \eqref{eq:newsolution1}, we have
\begin{align*}
    \tildere(s, a)
    &=
    \Omega'(s, a;\pie)
    -
    \Omega^*(\Omega'(s, \cdot;\pie))\\
    &=
    \Omega'(s, a;\pie)
    -
    \left\{
        \inner{\pie(\cdot|s)}{\Omega'(s, \cdot;\pie)}{\A}-\Omega(\pie(\cdot|s))
    \right\}
    \\
    &=
    \Omega'(s, a;\pie)
    -
    \E_{a'\sim\pie(\cdot|s)}[\Omega'(s, a';\pie)]
    +
    \Omega(\pie(\cdot|s))
    \\
    &=
    t(s, a;\pie),
\end{align*}
where $t(s, a;\pie)$ is our proposed solution in \lemmaref{lemma1}.
}

\section{Proof of \lemmaref{lemma3}}
\label{proofoflemma3}

\paragraph{RL objective in Regularized MDPs w.r.t. normalized visitation distributions.}
For a reward function $r\in \R^{\SA}$ and a strongly convex function $\Omega:\P^{\A}\rightarrow\R$, 
the RL objetive $J_\Omega(r, \pi)$ in~\eqref{RL_OBJECTIVE} is equivalent to 
\begin{align}
    \argmax_{\pi}
    \bar{J}_{\bOmega}(r, d_\pi)
    :=
    \inner{r}{d_\pi}{\SA}-\bOmega(d_\pi)
    ,
    \label{EQRLWITHD}
\end{align}
where for a set $\D$ of \emph{normalized} visitation distributions~\citep{syed2008apprenticeship}
\begin{align*}
    \D
    :=
    \left\{
        d\in\P^\SA:
        \sum_{a'}d(s', a')
        =
        (1-\gamma)P_0(s')
        +
        \gamma
        \sum_{s, a}
        P(s'|s, a)d(s, a),
        \forall s'\in\S
    \right\},
\end{align*}
we define
$
    \bOmega({d})
    :=\E_{(s, a)\sim d}[\Omega(\bar{\pi}_d(\cdot|s))]
$
and
$
    \bar{\pi}_d(\cdot|s)
    := \frac{{d}(s, \cdot)}{\sum_{a'}{d}(s,a')}
    \in\P_{\S}^{\A}
$
for $d\in\D$ and
use $\bar{\pi}_{d_\pi}(\cdot|s)=\pi(\cdot|s)$ for all $s\in\S$. For $\bOmega:\D\rightarrow\R$, its convex conjugate $\bOmega^*$ is
\begin{align}
    \bOmega^*(r)
    :&=
    \max_{d\in\D}
    \bar{J}_{\bOmega}(r, d)
    \nonumber
    \\
    &=
    \max_{d\in\D}
    \inner{r}{d}{\SA}
    -
    \bOmega(d)
    \nonumber
    \\
    &\overset{(\mathrm{i})}{=}
    \max_{\pi\in\P_\S^\A}
    \inner{r}{d_\pi}{\SA}
    -
    \bOmega(d_\pi)
    \nonumber
    \\
    &=
    \max_{\pi\in\P_\S^\A}
    \sum_{s, a}d_\pi(s, a)\left[r(s, a)-\Omega(\pi(a|s))\right]
    \nonumber
    \\
    &=
    (1-\gamma)\cdot
    \max_{\pi\in\P_\S^\A}
    J_\Omega(r, \pi),
    \label{EQLAST}
\end{align}
where $(\mathrm{i})$ follows from using the one-to-one correspondence between policies and visitation distributions~\citep{syed2008apprenticeship,ho2016generative}.
Note that \eqref{EQLAST} is equal to the optimal discounted average return in regularized MDPs.

\paragraph{IRL objective in Regularized MDPs w.r.t. normalized visitation distributions.}

By using the RL objective in \eqref{EQRLWITHD}, we can rewrite the IRL objective in \eqref{EQRIRL} w.r.t. the normalized visitation distributions as the maximization of the following objective over $r\in\R^\SA$:
\begin{align}
    &(1-\gamma)
    \cdot
    \left\{
        J_{\Omega}(r, \pie)
        -
        \max_{\pi\in\P_\S^\A} 
        J_{\Omega}(r, \pi)
    \right\}
    \nonumber
    \\
    &=
    \bar{J}_{\bOmega}(r, d_{\pie})
    -
    \max_{d\in\D} 
    \bar{J}_{\bOmega}(r, d)
    \nonumber
    \\
    &=
    \min_{d\in\D} 
    \left\{
        \bar{J}_{\bOmega}(r, d_{\pie})
        -
        \bar{J}_{\bOmega}(r, d)
    \right\}
    \nonumber
    \\
    &=
    \min_{d\in\D} 
    \left\{
        \left(
        \inner{r}{d_{\pie}}{\SA}
        -\bOmega(d_{\pie})
        \right)
        -
        \left(
        \inner{r}{d}{\SA}
        -\bOmega(d)
        \right)
    \right\}
    \nonumber
    \\
    &=
    \min_{d\in\D} 
    \left\{
        \bOmega(d)
        -
        \bOmega(d_{\pie})
        -
        \inner{r}{d-d_{\pie}}{\SA}
    \right\}.
    \label{EQIRLNEWOBJ}
\end{align}
Note that if $\nabla\bOmega(d)$ is well-defined and $r=\nabla\bOmega(d_{\pie})$
for any strictly convex $\bOmega$, 
\eqref{EQIRLNEWOBJ} is equal to
\begin{align*}
    &\min_{d\in\D} 
    \left\{
        \bOmega(d)
        -
        \bOmega(d_{\pie})
        -
        \inner{\nabla\bOmega(d_{\pie})}{d-d_{\pie}}{\SA}
    \right\}
    =
    \min_{d\in\D} 
    D_{\bOmega}^\SA(d||d_{\pie}),
\end{align*}
where the equality comes from the definition of Bregman divergence.

\paragraph{Proof of $t(s, a;\pi_d)=\nabla[\bOmega(d)](s, a)$.}

For simpler notation, we use matrix-vector notation for the proof when discrete state and action spaces $\S=\{1, ..., |\S|\}$ and $\A=\{1, ..., |\A|\}$ are considered. 
For a normalized visitation distribution $d\in\D$, let us define
\begin{align*}
    d^s_a
    &
    :=
    d(s, a), s\in\S, a\in\A,
    \\
    \pd^s
    &
    :=
    [d^s_1, ..., d^s_{|\A|}]^T\in\R^\A,
    s\in\S,
    \\
    \pD
    &
    :=
    [\pd^1, ..., \pd^{|\S|}]^T
    =
    \begin{bmatrix}
    d^1_1           &\cdots     &d^1_{|\A|}     \\
    \vdots          &\ddots     &\vdots         \\
    d^{|\S|}_1        &\cdots     &d^{|\S|}_{|\A|}  \\
    \end{bmatrix}
    \in\R^\SA,
    \\
    \ppi(\px)
    &
    :=
    \frac{\px}{\one{\A}^T\px}
    =\frac{1}{\sum_{a\in\A}x_a}
    \left[
        x_1, ..., x_{|\A|}
    \right]^T
    \in\R^\A,
    \px:=\left[
        x_1, ..., x_{|\A|}
    \right]^T\in\R^\A,
\end{align*}
where $\one{\A}=[1,...,1]^T\in\R^\A$ is an $|\A|$-dimensional all-one vector.
By using these notations, the original $\bOmega$ can be rewritten as
\begin{align*}
    \bOmega(\pD)
    =
    \sum_{s, a}d^s_a\Omega(\ppi(\pd^s))
    =
    \sum_{s\in\S}\one{\A}^T\pd^s\Omega(\ppi(\pd^s)).
\end{align*}
The gradient of $\bOmega$ w.r.t. $\pD$ (using denominator-layout notation) is
\begin{align*}
    \nabla_{\pD}\bOmega(\pD)
    &=
    \left[
        \dfrac{\partial\bOmega(\pD)}{\partial\pd^1},
        ...,
        \dfrac{\partial\bOmega(\pD)}{\partial\pd^{|\S|}}
    \right]^T\in\R^\SA,
\end{align*}
where each element of $\nabla_{\pD}\bOmega(\pD)$ satisfies
\begin{align}
    \dfrac{\partial\bOmega(\pD)}{\partial\pd^s}
    &=
    \left[
        \dfrac{\partial\bOmega(\pD)}{\partial d^s_1},
        ...,
        \dfrac{\partial\bOmega(\pD)}{\partial d^s_{|\A|}}
    \right]^T
    \nonumber
    \\
    &=
    \frac{\partial}{\partial\pd^s}
    \left\{
        \sum_{s\in\S}\one{\A}^T\pd^s\Omega(\ppi(\pd^s))
    \right\}
    \nonumber
    \\
    &=
    \Omega(\ppi(\pd^s))\one{\A}
    +
    \one{\A}^T\pd^s
    \frac{\partial\Omega(\ppi(\pd^s))}{\partial\pd^s}
    \nonumber
    \\
    &=
    \Omega(\ppi(\pd^s))\one{\A}
    +
    \one{\A}^T\pd^s
    \frac{\partial\ppi(\pd^s)}{\partial\pd^s}
    \frac{\partial\Omega(\ppi(\pd^s))}{\partial\ppi(\pd^s)}.
    \label{EQIDENTITY1}
\end{align}
for
\begin{align*}
    \frac{\partial\ppi(\pd^s)}{\partial\pd^s}
    &=
    \left[
        \frac{\partial\bpi_{1}(\pd^s)}{\partial\pd^s},
        ...,
        \frac{\partial\bpi_{|\A|}(\pd^s)}{\partial\pd^s}
    \right],
    \\
    \frac{\partial\bpi_{a}(\pd^s)}{\partial\pd^s}
    &=
    \frac{\partial}{\partial\pd^s}
    \left[
    \frac{d^s_a}{\one{\A}^T\pd^s}
    \right]
    =
    \frac{\partial d^s_a}{\partial\pd^s}
    (\one{\A}^T\pd^s)^{-1}
    +
    d^s_a
    \frac{\partial (\one{\A}^T\pd^s)^{-1}}{\partial\pd^s}.
\end{align*}
Note that each element of $\frac{\partial\bpi_{a}(\pd^s)}{\partial\pd^s}$ satisfies
\begin{align*}
    \frac{\partial\bpi_{a}(\pd^s)}{\partial d^s_{a'}}
    &=
    \frac{\partial d^s_a}{\partial d^s_{a'}}
    (\one{\A}^T\pd^s)^{-1}
    +
    d^s_a
    \frac{\partial (\one{\A}^T\pd^s)^{-1}}{\partial d^s_{a'}}
    \\
    &=
    \mathbb{I}\{a=a'\}
    (\one{\A}^T\pd^s)^{-1}
    -
    d^s_a
    (\one{\A}^T\pd^s)^{-2}
    \\
    &=
    \mathbb{I}\{a=a'\}
    (\one{\A}^T\pd^s)^{-1}
    -
    \bpi_a(\pd^s)
    (\one{\A}^T\pd^s)^{-1}
    \\
    &=
    (\one{\A}^T\pd^s)^{-1}
    \left[
        \mathbb{I}\{a=a'\}
        -
        \bpi_a(\pd^s)
    \right],
\end{align*}
and thus,
\begin{align}
    \frac{\partial\ppi(\pd^s)}{\partial\pd^s}
    &=
    (\one{\A}^T\pd^s)^{-1}
    \left\{
        \pI_{\A\times\A}
        -
        \one{\A}
        [\ppi(\pd^s)]^T
    \right\}.
    \label{EQIDENTITY2}
\end{align}
By substituting \eqref{EQIDENTITY2} into \eqref{EQIDENTITY1}, we have
\begin{align}
    \dfrac{\partial\bOmega(\pD)}{\partial\pd^s}
    &=
    \Omega(\ppi(\pd^s))\one{\A}
    +
    \one{\A}^T\pd^s
    \frac{\partial\ppi(\pd^s)}{\partial\pd^s}
    \frac{\partial\Omega(\ppi(\pd^s))}{\partial\ppi(\pd^s)}
    \nonumber
    \\
    &=
    \Omega(\ppi(\pd^s))\one{\A}
    +
    \one{\A}^T\pd^s
    \left[
        (\one{\A}^T\pd^s)^{-1}
        \left\{
            \pI_{\A\times\A}
            -
            \one{\A}
            [\ppi(\pd^s)]^T
        \right\}
    \right]
    \frac{\partial\Omega(\ppi(\pd^s))}{\partial\ppi(\pd^s)}
    \nonumber
    \\
    &=
    \Omega(\ppi(\pd^s))\one{\A}
    +
    \left\{
        \pI_{\A\times\A}
        -
        \one{\A}
        [\ppi(\pd^s)]^T
    \right\}
    \frac{\partial\Omega(\ppi(\pd^s))}{\partial\ppi(\pd^s)}
    \nonumber
    \\
    &=
    \Omega(\ppi(\pd^s))\one{\A}
    +
    \frac{\partial\Omega(\ppi(\pd^s))}{\partial\ppi(\pd^s)}
    -
    [\ppi(\pd^s)]^T
    \frac{\partial\Omega(\ppi(\pd^s))}{\partial\ppi(\pd^s)}
    \one{\A}
    \nonumber
    \\
    &=
    \frac{\partial\Omega(\ppi(\pd^s))}{\partial\ppi(\pd^s)}
    -
    [\ppi(\pd^s)]^T
    \frac{\partial\Omega(\ppi(\pd^s))}{\partial\ppi(\pd^s)}
    \one{\A}
    +
    \Omega(\ppi(\pd^s))\one{\A}.
    \label{EQIDENTITY3}
\end{align}
If we use the function notation, \eqref{EQIDENTITY3} can be written as
\begin{align*}
    \nabla[\bOmega(d)](s, a)
    &=
    \nabla\Omega(\bpi_d(\cdot|s))(a)
    -
    \E_{a'\sim\bpi_d(\cdot|s)}
    \left[
        \nabla\Omega(\bpi_d(\cdot|s))(a')
    \right]
    +
    \Omega(\bpi_d(\cdot|s))\\
    &=
    t(s, a;\bpi_d)
\end{align*}
for $t$ of \eqref{TARGET_REWARD} in \lemmaref{lemma1}.

\section{Derivation of Bregman-Divergence-Based Measure in Continuous Controls}

In \eqref{EQD}, the Bregman divergence in the control task is defined as
\begin{align}
D_\omega^{\A}(\mathbb{P}_1||\mathbb{P}_2)
:=
\int_\X
\left\{
    \omega(p_1(x))
    -
    \omega(p_2(x))
    -
    \omega'(p_2(x))
    (p_1(x)-p_2(x))
\right\}
dx.
\label{F:EQ1}
\end{align}
Note that we consider
$
    \Omega(p)
    =
    \int_\X\omega(p(x))dx
    =
    \int_\X\left[-f_\phi(p(x))\right]dx
$ for $f_\phi(x)=x\phi(x)$, which makes \eqref{F:EQ1} equal to
\begin{align*}
    &
    \int_\X
    \left\{
        -
        p_1(x)\phi(p_1(x))
        +
        p_2(x)\phi(p_2(x))
        +
        f_\phi'(p_2(x))
        (p_1(x) - p_2(x))
    \right\}
    dx
    \\
    &=
    \int_\X
    p_1(x)
    \left\{
        f_\phi'(p_2(x))
        -
        \phi(p_1(x))
    \right\}
    dx
    -
    \int_\X
    p_2(x)
    \left\{
        f_\phi'(p_2(x))
        -
        \phi(p_2(x))
    \right\}
    dx
    \\
    &=
    \E_{x\sim p_1}
    \left[
        f_\phi'(p_2(x))
        -
        \phi(p_1(x))
    \right]
    -
    \E_{x\sim p_2}
    \left[
        f_\phi'(p_2(x))
        -
        \phi(p_2(x))
    \right].
\end{align*}
Thus, by considering a learning agent's policy $\pi^s=\pi(\cdot|s)$, expert policy $\pie^s=\pie(\cdot|s)$, and the objective in \eqref{BREGMAN_OBJECTIVE} characterized by the Bregman divergence, we can think of the following measure between expert and agent policies:
\begin{align}
    &\E_{s\sim d_{\pi}}
    \left[
        D_\Omega^\A(\pi^s||\pie^s)
    \right]
    \nonumber
    \\
    &=
    \E_{s\sim d_{\pi}}
    \left[
        \E_{a\sim\pi^s}
        \left[
            f_\phi'(\pie^s(a))
            -
            \phi(\pi^s(a))
        \right]
        -
        \E_{a\sim\pie^s}
        \left[
            f_\phi'(\pie^s(a))
            -
            \phi(\pie^s(a))
        \right]
    \right].
    \label{F:EQ2}
\end{align}
\section{Tsallis entropy and associated Bregman divergence among multi-variate Gaussian distributions}
\label{appendix:tsallis}

Based on the derivation in~\citet{nielsen2011renyi}, we derive the Tsallis entropy and associated Bremgan divergence as follows. We first consider the distributions in the exponential family
\begin{align}
    \exp
    \left(
        \inner{\theta}{t(x)}{}
        -  
        F(\theta)
        +
        k(x)
    \right).
\end{align}
Note that for
\begin{align*}
    \theta 
    &=
    \begin{bmatrix}
        \Sigma^{-1}\mu
        \\
        -\frac{1}{2}\Sigma^{-1}
    \end{bmatrix}
    =
    \begin{bmatrix}
        \theta_1
        \\
        \theta_2
    \end{bmatrix},
    \\
    t(x)
    &=
    \begin{bmatrix}
        x
        \\
        xx^T
    \end{bmatrix},
    \\
    F(\theta)
    &=
    -\frac{1}{4}\theta_1^T\theta_2^{-1}\theta_1
    +
    \frac{1}{2}\log|-\pi\theta_2^{-1}|
    =
    \frac{1}{2}\mu^T\Sigma^{-1}\mu
    +
    \frac{1}{2}\log(2\pi)^d|\Sigma|,
    \\
    k(x)
    &=
    0,
\end{align*}
we can recover the multi-variate Gaussian distribution~\citep{nielsen2011renyi}:
\begin{align}
    &\exp(\inner{\theta}{t(x)}{}-F(\theta)+k(x))
    \\
    &=
    \exp
    \left(
        \mu^T\Sigma^{-1}x
        -\frac{1}{2}\mathrm{tr}(\Sigma^{-1}xx^T)
        -\frac{1}{2}\mu^T\Sigma^{-1}\mu
        -\frac{1}{2}\log(2\pi)^d|\Sigma|
    \right)\\
    &=
    \frac{1}{(2\pi)^{d/2}|\Sigma|^{1/2}}
    \exp
    \left(
        \mu^T\Sigma^{-1}x
        -\frac{1}{2}x^T\Sigma^{-1}x
        -\frac{1}{2}\mu^T\Sigma^{-1}\mu
    \right)
    \\
    &=
    \frac{1}{(2\pi)^{d/2}|\Sigma|^{1/2}}
    \exp
    \left(
        \frac{1}{2}
        (x-\mu)^T\Sigma^{-1}(x-\mu)
    \right).
\end{align}
For two distributions with $k(x)=0$, 
\begin{align*}
    \pi(x)
    =
    \exp(\inner{\theta}{t(x)}{}-F(\theta)),
    \hpi(x)
    =
    \exp(\inner{\htheta}{t(x)}{}-F(\htheta))
\end{align*}
that share $t, F$, and $k$, it can be shown that
\begin{align*}
    I(\pi, \hpi;\alpha, \beta)
    &=
    \int
    \pi(x)^\alpha \hpi(x)^\beta
    dx
    \\
    &=
    \exp
    \left(
        F(\alpha\theta+\beta\htheta)
        -
        \alpha
        F(\theta)
        -
        \beta
        F(\htheta)
    \right)
\end{align*}
since
\begin{align*}
    &
    \int
    \pi(x)^\alpha \hpi(x)^\beta
    dx
    \\
    =
    &
    \int
    \exp
    \left(
        \alpha\inner{\theta}{t(x)}{}-\alpha F(\theta)
        +
        \beta\inner{\htheta}{t(x)}{}-\beta F(\htheta)
    \right)
    dx
    \\
    =
    &
    \int
    \exp
    \left(
            \inner{\alpha\theta+\beta\htheta}{t(x)}{}
            -F(\alpha\theta+\beta\htheta)
    \right)
    \exp
    \left(
        F(\alpha\theta+\beta\htheta)
        -
        \alpha F(\theta) 
        -
        \beta F(\htheta)
    \right)
    dx
    \\
    =
    &
    \exp
    \left(
        F(\alpha\theta+\beta\htheta)
        -
        \alpha F(\theta) 
        -
        \beta F(\htheta)
    \right)
    \int
    \exp
    \left(
            \inner{\alpha\theta+\beta\htheta}{t(x)}{}
            -F(\alpha\theta+\beta\htheta)
    \right)
    dx
    \\
    =
    &
    \exp
    \left(
        F(\alpha\theta+\beta\htheta)
        -
        \alpha F(\theta) 
        -
        \beta F(\htheta)
    \right).
\end{align*}

\subsection{Tsallis Entropy}
For $\phi(x)=\frac{k}{q-1}(1-x^{q-1})$ and $k=1$, the Tsallis entropy of $\pi$ can be written as
\begin{align*}
    \mathcal{T}_q(\pi)
    :=
    \E_{x\sim\pi}\phi(x)
    &=
    \int\pi(x)\frac{1-\pi(x)^{q-1}}{q-1}dx
    \\
    &=
    \frac{1-\int \pi(x)^q dx}{q-1}
    \\
    &=
    \frac{1}{q-1}
    \left(
        1 - I(\pi, \pi; q, 0)
    \right)
    =
    \frac{
    1
    -
    \exp
    \left(
        F(q\theta)
        -
        q F(\theta)
    \right)
    }{
    q-1
    }.
\end{align*}
If $\pi$ is a multivariate Gaussian distribution, we have
\begin{align*}
    F(q \theta)
    &
    =
    \frac{q}{2}\mu^T\Sigma^{-1}\mu
    +
    \frac{1}{2}\log(2\pi)^d|\Sigma|
    -
    \frac{1}{2}\log q^d,\\
    q F(\theta)
    &
    =
    \frac{q}{2}\mu^T\Sigma^{-1}\mu
    +
    \frac{q}{2}\log(2\pi)^d|\Sigma|,\\
    F(q \theta)-q F(\theta)
    &
    =
    \frac{1-q}{2}\log(2\pi)^d|\Sigma|
    -
    \frac{1}{2}\log q^d
    \\
    &
    =
    (1-q)
    \left\{
        \frac{d}{2}\log2\pi
        +
        \frac{1}{2}\log|\Sigma|
        -
        \frac{d\log q}{2(1-q)}
    \right\}.
\end{align*}
For $\Sigma=\mathrm{diag}\{\sigma_1^2, ..., \sigma_d^2\}$, we have
\begin{align*}
    F(q \theta)-q F(\theta)
    &
    =
    (1-q)
    \left\{
        \frac{d}{2}\log2\pi
        +
        \frac{1}{2}\log|\Sigma|
        -
        \frac{d\log q}{2(1-q)}
    \right\}
    \\
    &
    =
    (1-q)
    \left\{
        \frac{d}{2}\log2\pi
        +
        \frac{1}{2}\log\prod_{i=1}^d\sigma_i^2
        -
        \frac{d\log q}{2(1-q)}
    \right\}
    \\
    &
    =
    (1-q)
    \sum_{i=1}^d
    \left\{
        \frac{\log2\pi}{2}
        +
        \log\sigma_i
        -
        \frac{\log q}{2(1-q)}
    \right\}.
\end{align*}

\subsection{\Blue{Tractable} form of baseline}

For $\phi(x)=\frac{k}{q-1}(1-x^{q-1})$, we have
\begin{align*}
    f_\phi'(x)
    &=
    \frac{k}{q-1}(1-qx^{q-1})\\
    &=
    \frac{k}{q-1}(q-qx^{q-1}-(q-1))\\
    &=
    \frac{qk}{q-1}(1-x^{q-1})-k\\
    &=
    q\phi(x)-k.
\end{align*}
Therefore, the baseline can be rewritten as
\begin{align*}
    \E_{x\sim\pi}[-f_\phi'(x)+\phi(x)]
    =
    \E_{x\sim\pi}[k - q\phi(x) + \phi(x)]
    =
    (1-q)\mathcal{T}_q(\pi)+k.
\end{align*}
For a multivariate Gaussian distribution $\pi$, the \Blue{tractable} form of $\E_{x\sim\pi}[-f_\phi'(x)+\phi(x)]$ can be derived by using that of Tsallis entropy $\mathcal{T}_q(\pi)$ of $\pi$.

\subsection{Bregman divergence with Tsallis entropy regularization}
\label{bregman_tsallis}
In \eqref{F:EQ2}, we consider the following form of the Bregman divergence:
\begin{align*}
    &\int
    \pi(x)
    \{f_\phi'(\hpi(x))-\phi(\pi(x))\}
    dx
    -
    \int
    \hpi(x)
    \{f_\phi'(\hpi(x))-\phi(\hpi(x))\}
    dx.
\end{align*}
For 
$\phi(x)=\frac{k}{q-1}(1-x^{q-1})$, 
$f_\phi'(x)=\frac{k}{q-1}(1-qx^{q-1})=q\phi(x)-k$, and $k=1$, 
the above form is equal to
\begin{align*}
    &
    \int
    \pi(x)
    \left[
        \frac{1-q\hpi(x)^{q-1}}{q-1}
    \right]
    dx
    -
    \mathcal{T}_q(\pi)
    -
    (q-1)\mathcal{T}_q(\hpi)+1
    \\
    &
    =
    \frac{1}{q-1}
    -
    \frac{q}{q-1}\int\pi(x)\hpi(x)^{q-1}dx
    -
    \mathcal{T}_q(\pi)
    -
    (q-1)\mathcal{T}_q(\hpi)+1
    \\
    &
    =
    \frac{q}{q-1}
    -
    \frac{q}{q-1}\int\pi(x)\hpi(x)^{q-1}dx
    -
    \mathcal{T}_q(\pi)
    -
    (q-1)\mathcal{T}_q(\hpi).
\end{align*}
For multivariate Gaussians
\begin{align*}
    \pi(x)
    &=
    \mathcal{N}
    (x;\mu, \Sigma), 
    \mu=[\nu_1, ..., \nu_d]^T,
    \Sigma=\mathrm{diag}(\sigma_1^2, ..., \sigma_d^2),\\
    \hpi(x)
    &=
    \mathcal{N}
    (x;\hmu, \hSigma), 
    \hmu=[\hnu_1, ..., \hnu_d]^T,
    \hSigma=\mathrm{diag}(\hsigma_1^2, ..., \hsigma_d^2),
\end{align*}
we have
\begin{align*}
    \int\pi(x)\hpi(x)^{q-1}dx
    &
    =
    I(\pi,\hpi;1, q-1)
    =
    \exp
    \left(
        F(\theta')
        -
        F(\theta)
        -
        (q-1)F(\htheta)
    \right),
\end{align*}
where
\begin{align*}
    \theta 
    &=
    \begin{bmatrix}
        \Sigma^{-1}\mu
        \\
        -\frac{1}{2}\Sigma^{-1}
    \end{bmatrix},
    \\
    \htheta 
    &=
    \begin{bmatrix}
        \hSigma^{-1}\hmu
        \\
        -\frac{1}{2}\hSigma^{-1}
    \end{bmatrix}
    ,
    \\
    \theta'
    =\theta+(q-1)\htheta 
    &=
    \begin{bmatrix}
        \Sigma^{-1}\mu
        +
        (q-1)
        \hSigma^{-1}\hmu
        \\
        -\frac{1}{2}
        (
            \Sigma^{-1}
            +
            (q-1)
            \hSigma^{-1}
        )
    \end{bmatrix}
    =
    \begin{bmatrix}
        \theta_1'
        \\
        \theta_2'
    \end{bmatrix}
    ,\\
    \theta_1'
    &=
    \left[
        \frac{\nu_1}{\sigma_1^2}+(q-1)\frac{\hnu_1}{\hsigma_1^2},
        ...,
        \frac{\nu_d}{\sigma_d^2}+(q-1)\frac{\hnu_d}{\hsigma_d^2}
    \right]^T,\\
    \theta_2'
    &=
    -\frac{1}{2}
    \mathrm{diag}
    \left\{
        \frac{1}{\sigma_1^2}+(q-1)\frac{1}{\hsigma_1^2},
        ...,
        \frac{1}{\sigma_d^2}+(q-1)\frac{1}{\hsigma_d^2}
    \right\},\\
\end{align*}
and
\begin{align*}
    F(\theta)
    &
    =
    \frac{1}{2}\mu^T\Sigma^{-1}\mu
    +
    \frac{1}{2}\log(2\pi)^d|\Sigma|
    =
    \sum_{i=1}^d
    \left\{
        \frac{\nu_i^2}{2\sigma_i^2}
        +
        \frac{\log2\pi}{2}
        +
        \log\sigma_i
    \right\}
    ,
    \\
    F(\htheta)
    &
    =
    \frac{1}{2}\hmu^T\hSigma^{-1}\hmu
    +
    \frac{1}{2}\log(2\pi)^d|\hSigma|
    =
    \sum_{i=1}^d
    \left\{
        \frac{\hnu_i^2}{2\hsigma_i^2}
        +
        \frac{\log2\pi}{2}
        +
        \log\hsigma_i
    \right\},
    \\
    F(\theta + (q-1) \htheta)
    &=
    -\frac{1}{4}
    (\theta_1')^T(\theta_2')^{-1}\theta_1'
    +
    \frac{1}{2}\log|-\pi(\theta_2')^{-1}|
    \\
    &=
    \sum_{i=1}^d
    \left\{
        \frac{1}{2}
        \frac{
            \left(
                \frac{\nu_i}{\sigma_i^2}+(q-1)\frac{\hnu_i}{\hsigma_i^2}
            \right)^2
        }{
            \frac{1}{\sigma_i^2}+(q-1)\frac{1}{\hsigma_i^2}
        }
        +
        \frac{\log2\pi}{2}
        +
        \log
        \frac{
            1
        }{
            \frac{1}{\sigma_i^2}+(q-1)\frac{1}{\hsigma_i^2}
        }
    \right\}. 
\end{align*}

\newpage

\section{Experiment Setting}
\label{experiment_setting}

\subsection{Policy Regularizers in Experiments}
\begin{center}
\begin{table}[h]
\small
\caption{
Policy regularizers $\phi$ and their corresponding $f_\phi$~\citep{yang2019regularized}.
}
\centering
\scriptsize
\begin{tabular}{cccc}
\toprule
  reg. type.
& condition
& $\phi(x)$ 
& $f_\phi'(x)$\\
\midrule\midrule\xrowht[()]{10pt}
  $\mathrm{Shannon}$ 
& -
& $-\log x$
& $-\log x-1$
\\\midrule\xrowht[()]{10pt}
  $\mathrm{Tsallis}$ 
& $k > 0, q > 1$
& $\frac{k}{q-1}(1-x^{q-1})$
& $\frac{k}{q-1}(1-qx^{q-1})$
\\\hline\xrowht[()]{10pt}
  $\mathrm{Exp}$
& $k\ge0, q\ge1$
& $q - x^kq^x$
& $q - x^kq^x(k + 1 + x\log q)$
\\\hline\xrowht[()]{10pt}
  $\mathrm{Cos}$
& $0<\theta\le\pi/2$
& $\cos(\theta x) - \cos(\theta)$
& $-\cos(\theta)+\cos(\theta x)-\theta x \sin(\theta x)$
\\\hline\xrowht[()]{10pt}
  $\mathrm{Sin}$
& $0<\theta\le\pi/2$
& $\sin(\theta) - \sin(\theta x)$ 
& $\sin(\theta)-\sin(\theta x) - \theta x \cos(\theta x)$
\\\bottomrule
\end{tabular}
\label{table:reg}
\end{table}
\end{center}

\Blue{
\subsection{Density-based model}
\label{sec:dbm}

By exploiting the knowledge on the reward in \corollaryref{corollary1}
\begin{align*}
    -f_\phi'(\pi(a|s))
    -
    \E_{a'\sim\pi(\cdot|s)}
    [f_\phi'(\pi(a'|s))-\phi(\pi(a'|s))],
\end{align*}
we consider the density-based model (DBM) which is defined by
\begin{align*}
    r_\theta(s, a)\approx-f_\phi'(\pi_{\theta_1}(a|s))+B_{\theta_2}(s)
\end{align*}
for $\theta=(\theta_1, \theta_2)$. Here, $f_\phi'$ is a function that can be known priorly, and $\pi_{\theta_1}$ is a neural network which is defined separately from the policy neural network $\pi_{\psi}$. The DBM for discrete control problems are depicted in Figure~\ref{modelfigure} (\emph{Left}). The model outputs rewards over all actions in parallel, where softmax is used for $\pi_{\theta_1}(\cdot|s)$ and $-f_\phi'$ is elementwisely applied to those softmax outputs followed by elementwisely adding $B_{\theta_2}(s)$.
For continuous control (Figure~\ref{modelfigure}, \emph{Right}), we use the network architecture similar to that in discrete control, where multivariate Gaussian distribution is used instead of softmax layer.

\begin{figure}[h]
\centering
\includegraphics[width=0.3\textwidth]{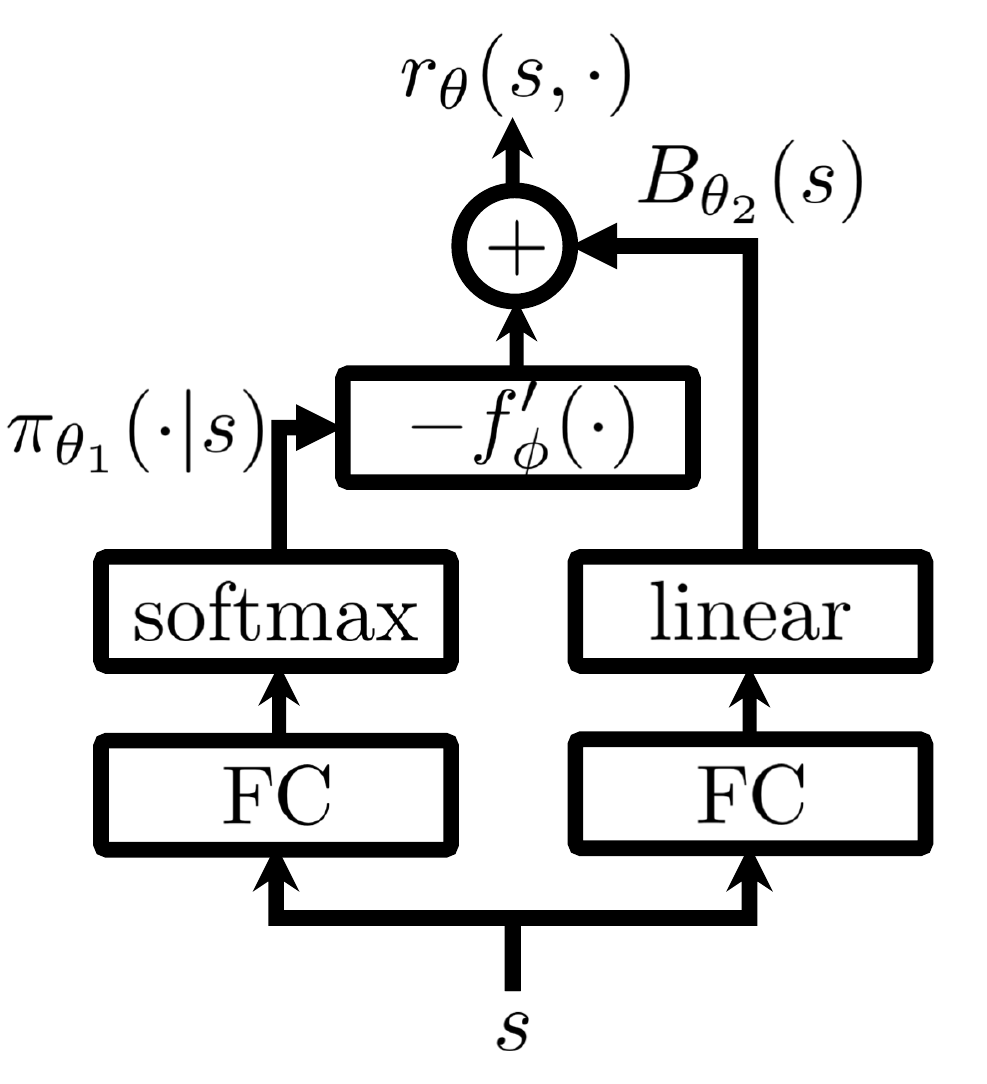}
\hspace{0.2in}
\includegraphics[width=0.3\textwidth]{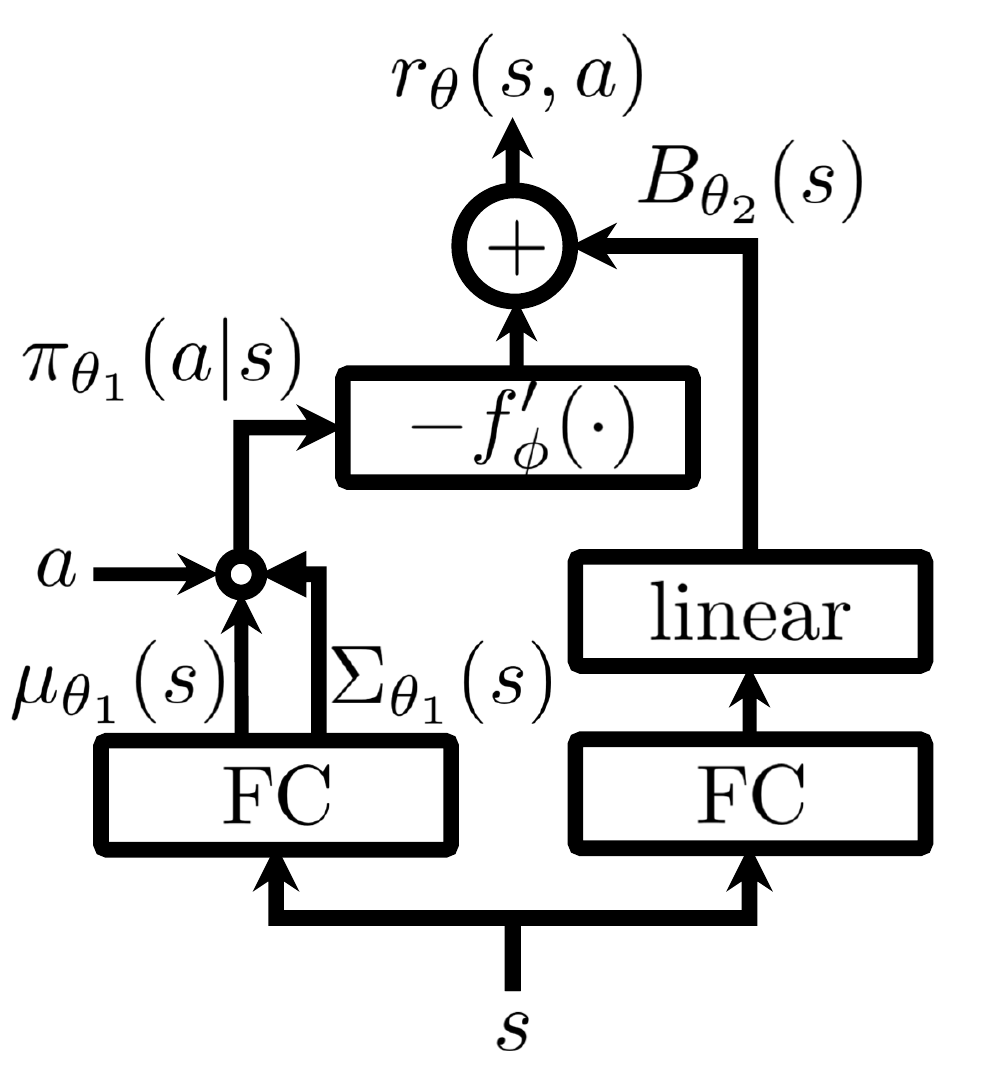}
\caption{Density-based model for discrete (\emph{Left}) and continuous control (\emph{Right})}
\label{modelfigure}
\end{figure}
}

\subsection{Expert in Bermuda World environment}
\label{expert_in_bermuda}

We assume a stochastic expert defined by 
\begin{gather*}
    \pie(a|s)
    =
    \frac{
    \sum_{t=1}^3
    (d^{(t)})^{-1}
    \I\{
    a=\mathrm{Proj}(\theta^{(t)})
    \}
    }{
    \sum_{t=1}^3
    (d^{(t)})^{-1}
    },\\
    \theta^{(t)}=\arctantwo(
        \bar{y}^{(t)}-y,
        \bar{x}^{(t)}-x
    ),
    d^{(t)}=\|\bar{s}^{(t)}-s\|_2^4 + \epsilon, t=1, 2, 3,
\end{gather*}
for 
$
    s=(x, y), 
    \bar{s}^{(1)}
    =
    (
        \bar{x}^{(1)}, 
        \bar{y}^{(1)}
    )=(-5, 10), 
    \bar{s}^{(2)}
    =
    (
        \bar{x}^{(2)}, 
        \bar{y}^{(2)}
    )=(0, 10), 
    \bar{s}^{(3)}
    =
    (
        \bar{x}^{(3)}, 
        \bar{y}^{(3)}
    )=(5, 10)
$, 
$\epsilon=10^{-4}$ and 
an operator $\mathrm{Proj}(\theta):\R\rightarrow\A$ that maps $\theta$ to the closest angle in $\A$.
In Figure~\ref{expert policy}, we depicted the expert policy.
\begin{center}
    \begin{figure}[h]
        \centering
        \includegraphics[height=0.23\textwidth]{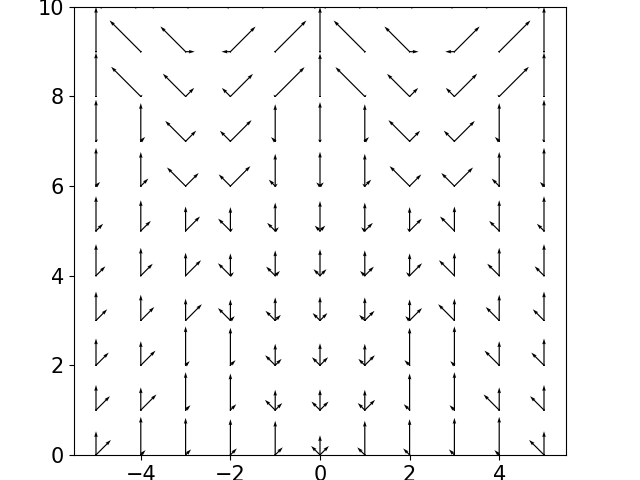}
        \caption{
        Visualization of the expert policy
        }
        \label{expert policy}
    \end{figure}
    \vspace{-20pt}
\end{center}

\subsection{MuJoCo experiment setting}\label{appendix_mujoco}

Instead of directly using MuJoCo environments with $\mathrm{tanh}$-squashed policies proposed in \Blue{Soft Actor-Critic (SAC)~\citep{haarnoja2018soft}}, we move $\tanh$ to a part of environment---named \emph{hyperbolized environments} in short---and assume Gaussian policies. Specifically, after an action $a$ is sampled from the policies, we pass $\mathrm{tanh}(a)$ to the environment.
We then consider multi-variate Gaussian policy
\begin{align*} 
\pi(\cdot|s)=\mathcal{N}\left(\pmb{\mu}(s), \pmb{\Sigma}(s)\right)
\end{align*}
with $\pmb{\mu}(s)=[\mu_1(s), ..., \mu_d(s)]^T$, $\pmb{\Sigma}(s)=\mathrm{diag}\{(\sigma_1(s))^2, ..., (\sigma_d(s))^2\}$, where
\begin{align*}
    -\mathrm{arctanh}(0.99)\le\mu_i(s)\le\mathrm{arctanh}(0.99),
    \log(0.01)\le\log\sigma_i(s)\le\log(2)
\end{align*}
for all $i=1, ..., d$.
Instead of using clipping, we use $\mathrm{tanh}$-activated outputs and scale them to be fit in the above ranges, which empirically improves the performance. 
Also, instead of using potential-based reward shaping used in AIRL~\citep{fu2018learning}, we update moving mean of intermediate reward values and update \Blue{the} value \Blue{network} with mean-subtracted rewards\Blue{---so that the value network gets approximately mean-zero reward---}to stabilize RL \Blue{part of RAIRL.} Note that this is motivated by \lemmaref{lemma2} from which we can guarantee that any constant shift of reward functions does not change optimality.

\newpage
\subsection{Hyperparameters}
Table \ref{table:fixed_hyperparams_Bandit}, Table \ref{table:fixed_hyperparams_BermudaWorld} and Table \ref{table:fixed_hyperparams_MuJoCo} list the parameters used in our Bandit, Bermuda World, and MuJoCo experiments, respectively.

\begin{table}[h]
\small
\centering
\caption{Hyperparameters for Bandit environments.}
\begin{tabular}{lc}
\label{table:fixed_hyperparams_Bandit}
\\ \toprule
Hyper-parameter                             &Bandit 
\\ \hline\hline
Batch size                     &{ 500 } \\
Initial exploration steps
                                            &{ 10,000 } \\
Replay size                    &{ 500,000 } \\
Target update rate ($\tau$)    &{ 0.0005 } \\
Learning rate                  &{ 0.0005 } \\
$\lambda$                      &{ 5 } \\
$q$ (Tsallis entropy $\mathcal{T}_q^k$)
                                            &{ 2.0 } \\
$k$ (Tsallis entropy $\mathcal{T}_q^k$)
                                            &{ 1.0 } \\
 Number of trajectories         &{ 1,000 } \\
Reward learning rate           &{ 0.0005 } \\
Steps per update       
                                            &{ 50 } \\
Total environment steps        &{ 500,000 } \\
\bottomrule
\end{tabular}
\end{table}

\begin{table}[h]
\small
\centering
\caption{Hyperparameters for Bermuda World environment.}
\begin{tabular}{lc}
\label{table:fixed_hyperparams_BermudaWorld}
\\ \toprule
Hyper-parameter                             &Bermuda World
\\ \hline\hline
Batch size                     &{ 500 } \\
Initial exploration steps         
                                            &{ 10,000 } \\
Replay size                    &{ 500,000 } \\
Target update rate ($\tau$)     
                                            &{ 0.0005 } \\
Learning rate                  &{ 0.0005 } \\
$q$ (Tsallis entropy $\mathcal{T}_q^k$)
                                            &{ 2.0 } \\
$k$ (Tsallis entropy $\mathcal{T}_q^k$)
                                            &{ 1.0 } \\
 Number of trajectories         &{ 1,000 } \\
Reward learning rate           &{ 0.0005 } \\
(For evaluation) $\lambda$
                                            &{ 1 } \\
(For evaluation) Learning rate             
                                            &{ 0.001 } \\
(For evaluation) Target update rate ($\tau$)
                                            &{ 0.0005 } \\
Steps per update   
                                            &{ 50 } \\
Number of steps                &{ 500,000 } \\
\bottomrule
\end{tabular}
\end{table}

\begin{table}[h]
\small
\centering
\caption{Hyperparameters for MuJoCo environments.}
\begin{tabular}{lcccccc}
\label{table:fixed_hyperparams_MuJoCo}
\\ \toprule
Hyper-parameter                             &Hopper             &Walker2d           &HalfCheetah            &Ant
\\ \hline
\\
Batch size                     &{ 256 }            &{ 256 }            &{ 256 }                &{ 256 } \\
Initial exploration steps
                                            &{ 10,000 }         &{ 10,000 }         &{ 10,000 }             &{ 10,000 } \\
Replay size                    &{ 1,000,000 }      &{ 1,000,000 }      &{ 1,000,000 }          &{ 1,000,000 } \\
Target update rate ($\tau$)    &{ 0.005 }          &{ 0.005 }          &{ 0.005 }              &{ 0.005 } \\
Learning rate                  &{ 0.001 }          &{ 0.001 }          &{ 0.001 }              &{ 0.001 } \\
$\lambda$                         &{0.0001}           &{0.000001}         &{ 0.0001 }             &{ 0.000001 } \\
$k$ (Tsallis entropy $\mathcal{T}_q^k$)
                                            &{ 1.0 }            &{ 1.0 }            &{ 1.0 }                &{ 1.0 } \\
 Number of trajectories         &{ 100 }            &{ 100 }              &{ 100 }              &{ 100 } \\
Reward learning rate           &{ 0.001 }          &{ 0.001 }            &{ 0.001 }            &{ 0.001 } \\
Steps per update   
                                            &{ 1 }              &{ 1 }              &{ 1 }                  &{ 1 } \\
Number of steps                &{ 1,000,000 }      &{ 1,000,000 }      &{ 1,000,000 }          &{2,000,000} \\       
\bottomrule
\end{tabular}
\end{table}

\end{document}

%% file: math_commands.tex
\usepackage{amsmath,amsfonts,bm}

\def\eqref#1{equation~\ref{#1}}

\def\1{\bm{1}}

\def\re{{\textnormal{e}}}

\DeclareMathAlphabet{\mathsfit}{\encodingdefault}{\sfdefault}{m}{sl}
\SetMathAlphabet{\mathsfit}{bold}{\encodingdefault}{\sfdefault}{bx}{n}

\newcommand{\E}{\mathbb{E}}

\newcommand{\R}{\mathbb{R}}

\DeclareMathOperator*{\argmax}{arg\,max}
\DeclareMathOperator*{\argmin}{arg\,min}

%% file: arXiv.bbl
\begin{thebibliography}{38}
\providecommand{\natexlab}[1]{#1}
\providecommand{\url}[1]{\texttt{#1}}
\expandafter\ifx\csname urlstyle\endcsname\relax
  \providecommand{\doi}[1]{doi: #1}\else
  \providecommand{\doi}{doi: \begingroup \urlstyle{rm}\Url}\fi

\bibitem[Amari(2009)]{amari2009alpha}
Shun-Ichi Amari.
\newblock $\alpha$-divergence is unique, belonging to both $f$-divergence and
  bregman divergence classes.
\newblock \emph{IEEE Transactions on Information Theory}, 55\penalty0
  (11):\penalty0 4925--4931, 2009.

\bibitem[Boularias \& Chaib-Draa(2010)Boularias and
  Chaib-Draa]{boularias2010bootstrapping}
Abdeslam Boularias and Brahim Chaib-Draa.
\newblock Bootstrapping apprenticeship learning.
\newblock In \emph{Advances in Neural Information Processing Systems
  (NeurIPS)}, pp.\  289--297, 2010.

\bibitem[Bregman(1967)]{bregman1967relaxation}
Lev~M Bregman.
\newblock The relaxation method of finding the common point of convex sets and
  its application to the solution of problems in convex programming.
\newblock \emph{USSR computational mathematics and mathematical physics},
  7\penalty0 (3):\penalty0 200--217, 1967.

\bibitem[Csisz\'ar(1963)]{csiszar63}
Imre Csisz\'ar.
\newblock Eine informationstheoretische ungleichung und ihre anwendung auf den
  beweis der ergodizitat von markoffschen ketten.
\newblock \emph{Magyar. Tud. Akad. Mat. Kutat\'o Int. K\"ozl}, 8:\penalty0
  85--108, 1963.

\bibitem[Dadashi et~al.(2020)Dadashi, Hussenot, Geist, and
  Pietquin]{dadashi2020primal}
Robert Dadashi, L{\'e}onard Hussenot, Matthieu Geist, and Olivier Pietquin.
\newblock Primal wasserstein imitation learning.
\newblock \emph{arXiv preprint arXiv:2006.04678}, 2020.

\bibitem[Finn et~al.(2016{\natexlab{a}})Finn, Christiano, Abbeel, and
  Levine]{finn2016connection}
Chelsea Finn, Paul Christiano, Pieter Abbeel, and Sergey Levine.
\newblock A connection between generative adversarial networks, inverse
  reinforcement learning, and energy-based models.
\newblock \emph{arXiv preprint arXiv:1611.03852}, 2016{\natexlab{a}}.

\bibitem[Finn et~al.(2016{\natexlab{b}})Finn, Levine, and
  Abbeel]{finn2016guided}
Chelsea Finn, Sergey Levine, and Pieter Abbeel.
\newblock Guided cost learning: {D}eep inverse optimal control via policy
  optimization.
\newblock In \emph{Proceedings of the 33rd International Conference on Machine
  Learning (ICML)}, pp.\  49--58, 2016{\natexlab{b}}.

\bibitem[Fu et~al.(2018)Fu, Luo, and Levine]{fu2018learning}
Justin Fu, Katie Luo, and Sergey Levine.
\newblock Learning robust rewards with adverserial inverse reinforcement
  learning.
\newblock In \emph{Proceedings of the 6th International Conference on Learning
  Representations (ICLR)}, 2018.

\bibitem[Fujimoto et~al.(2018)Fujimoto, Hoof, and
  Meger]{fujimoto2018addressing}
Scott Fujimoto, Herke Hoof, and David Meger.
\newblock Addressing function approximation error in actor-critic methods.
\newblock In \emph{Proceedings of the 35th International Conference on Machine
  Learning (ICML)}, pp.\  1582--1591, 2018.

\bibitem[Geist et~al.(2019)Geist, Scherrer, and Pietquin]{geist2019theory}
Matthieu Geist, Bruno Scherrer, and Olivier Pietquin.
\newblock A theory of regularized {M}arkov decision processes.
\newblock In \emph{Proceedings of the 36th International Conference on Machine
  Learning (ICML)}, pp.\  2160--2169, 2019.

\bibitem[Ghasemipour et~al.(2019)Ghasemipour, Zemel, and
  Gu]{ghasemipour2019divergence}
Seyed Kamyar~Seyed Ghasemipour, Richard Zemel, and Shixiang Gu.
\newblock A divergence minimization perspective on imitation learning methods.
\newblock In \emph{Proceedings of the 3rd Conference on Robot Learning (CoRL)},
  2019.

\bibitem[Goodfellow et~al.(2014)Goodfellow, {P}ouget {A}badie, Mirza, Xu,
  {W}arde {F}arley, Ozair, Courville, and Bengio]{goodfellow2014generative}
Ian Goodfellow, Jean {P}ouget {A}badie, Mehdi Mirza, Bing Xu, David {W}arde
  {F}arley, Sherjil Ozair, Aaron Courville, and Yoshua Bengio.
\newblock Generative adversarial nets.
\newblock In \emph{Advances in Neural Information Processing Systems
  (NeurIPS)}, pp.\  2672--2680, 2014.

\bibitem[Guo et~al.(2017)Guo, Hong, and Yang]{guo2017ambiguity}
Xin Guo, Johnny Hong, and Nan Yang.
\newblock Ambiguity set and learning via {B}regman and {W}asserstein.
\newblock \emph{arXiv preprint arXiv:1705.08056}, 2017.

\bibitem[Haarnoja et~al.(2018)Haarnoja, Zhou, Abbeel, and
  Levine]{haarnoja2018soft}
Tuomas Haarnoja, Aurick Zhou, Pieter Abbeel, and Sergey Levine.
\newblock Soft {A}ctor-{C}ritic: Off-policy maximum entropy deep reinforcement
  learning with a stochastic actor.
\newblock In \emph{Proceedings of the 35th International Conference on Machine
  Learning (ICML)}, pp.\  1861--1870, 2018.

\bibitem[Ho \& Ermon(2016)Ho and Ermon]{ho2016generative}
Jonathan Ho and Stefano Ermon.
\newblock Generative adversarial imitation learning.
\newblock In \emph{Advances in Neural Information Processing Systems
  (NeurIPS)}, pp.\  4565--4573, 2016.

\bibitem[Jones \& Byrne(1990)Jones and Byrne]{jones1990general}
Lee~K Jones and Charles~L Byrne.
\newblock General entropy criteria for inverse problems, with applications to
  data compression, pattern classification, and cluster analysis.
\newblock \emph{IEEE Transactions on Information Theory}, 36\penalty0
  (1):\penalty0 23--30, 1990.

\bibitem[Ke et~al.(2019)Ke, Barnes, Sun, Lee, Choudhury, and
  Srinivasa]{ke2019imitation}
Liyiming Ke, Matt Barnes, Wen Sun, Gilwoo Lee, Sanjiban Choudhury, and
  Siddhartha Srinivasa.
\newblock Imitation learning as $f$-divergence minimization.
\newblock \emph{arXiv preprint arXiv:1905.12888}, 2019.

\bibitem[Lee et~al.(2018)Lee, Choi, and Oh]{lee2018maximum}
Kyungjae Lee, Sungjoon Choi, and Songhwai Oh.
\newblock Maximum causal tsallis entropy imitation learning.
\newblock In \emph{Advances in Neural Information Processing Systems
  (NeurIPS)}, pp.\  4403--4413, 2018.

\bibitem[Lee et~al.(2019)Lee, Kim, Lim, Choi, and Oh]{lee2019tsallis}
Kyungjae Lee, Sungyub Kim, Sungbin Lim, Sungjoon Choi, and Songhwai Oh.
\newblock Tsallis reinforcement learning: A unified framework for maximum
  entropy reinforcement learning.
\newblock \emph{arXiv preprint arXiv:1902.00137}, 2019.

\bibitem[Lee et~al.(2020)Lee, Kim, Lim, Choi, Hong, Kim, Park, and
  Oh]{lee2020generalized}
Kyungjae Lee, Sungyub Kim, Sungbin Lim, Sungjoon Choi, Mineui Hong, Jaein Kim,
  Yong-Lae Park, and Songhwai Oh.
\newblock Generalized {T}sallis entropy reinforcement learning and its
  application to soft mobile robots.
\newblock Robotics: Science and Systems Foundation, 2020.

\bibitem[Li et~al.(2018)Li, Rath, and Burdick]{li2018inverse}
Kun Li, Mrinal Rath, and Joel~W Burdick.
\newblock Inverse reinforcement learning via function approximation for
  clinical motion analysis.
\newblock In \emph{2018 IEEE International Conference on Robotics and
  Automation (ICRA)}, pp.\  610--617. IEEE, 2018.

\bibitem[Mescheder et~al.(2017)Mescheder, Nowozin, and
  Geiger]{mescheder2017adversarial}
Lars Mescheder, Sebastian Nowozin, and Andreas Geiger.
\newblock Adversarial variational {B}ayes: Unifying variational autoencoders
  and generative adversarial networks.
\newblock In \emph{Proceedings of the 34th International Conference on Machine
  Learning (ICML)}, pp.\  2391--2400, 2017.

\bibitem[Mnih et~al.(2015)Mnih, Kavukcuoglu, Silver, Rusu, Veness, Bellemare,
  Graves, Riedmiller, Fidjeland, Ostrovski, et~al.]{mnih2015human}
Volodymyr Mnih, Koray Kavukcuoglu, David Silver, Andrei~A Rusu, Joel Veness,
  Marc~G Bellemare, Alex Graves, Martin Riedmiller, Andreas~K Fidjeland, Georg
  Ostrovski, et~al.
\newblock Human-level control through deep reinforcement learning.
\newblock \emph{Nature}, 518\penalty0 (7540):\penalty0 529--533, 2015.

\bibitem[Mnih et~al.(2016)Mnih, Badia, Mirza, Graves, Lillicrap, Harley,
  Silver, and Kavukcuoglu]{mnih2016asynchronous}
Volodymyr Mnih, Adria~Puigdomenech Badia, Mehdi Mirza, Alex Graves, Timothy
  Lillicrap, Tim Harley, David Silver, and Koray Kavukcuoglu.
\newblock Asynchronous methods for deep reinforcement learning.
\newblock In \emph{Proceedings of the 33rd International Conference on Machine
  Learning (ICML)}, pp.\  1928--1937, 2016.

\bibitem[Ng et~al.(1999)Ng, Harada, and Russell]{ng1999policy}
Andrew~Y Ng, Daishi Harada, and Stuart Russell.
\newblock Policy invariance under reward transformations: Theory and
  application to reward shaping.
\newblock In \emph{Proceedings of the 16th International Conference on Machine
  Learning (ICML)}, pp.\  278--287, 1999.

\bibitem[Ng et~al.(2000)Ng, Russell, et~al.]{ng2000algorithms}
Andrew~Y Ng, Stuart~J Russell, et~al.
\newblock Algorithms for inverse reinforcement learning.
\newblock In \emph{Proceedings of the 17th International Conference on Machine
  Learning (ICML)}, pp.\  663--670, 2000.

\bibitem[Nielsen \& Nock(2011)Nielsen and Nock]{nielsen2011renyi}
Frank Nielsen and Richard Nock.
\newblock On {R}\'enyi and {T}sallis entropies and divergences for exponential
  families.
\newblock \emph{arXiv preprint arXiv:1105.3259}, 2011.

\bibitem[Russell(1998)]{russell1998learning}
Stuart Russell.
\newblock Learning agents for uncertain environments.
\newblock In \emph{Proceedings of the 11th Annual Conference on Computational
  Learning Theory}, pp.\  101--103, 1998.

\bibitem[Schulman et~al.(2015)Schulman, Levine, Abbeel, Jordan, and
  Moritz]{schulman2015trust}
John Schulman, Sergey Levine, Pieter Abbeel, Michael Jordan, and Philipp
  Moritz.
\newblock Trust region policy optimization.
\newblock In \emph{Proceedings of the 32nd International Conference on Machine
  Learning (ICML)}, pp.\  1889--1897, 2015.

\bibitem[Schulman et~al.(2017)Schulman, Wolski, Dhariwal, Radford, and
  Klimov]{schulman2017proximal}
John Schulman, Filip Wolski, Prafulla Dhariwal, Alec Radford, and Oleg Klimov.
\newblock Proximal policy optimization algorithms.
\newblock \emph{arXiv preprint arXiv:1707.06347}, 2017.

\bibitem[Sharifzadeh et~al.(2016)Sharifzadeh, Chiotellis, Triebel, and
  Cremers]{sharifzadeh2016learning}
Sahand Sharifzadeh, Ioannis Chiotellis, Rudolph Triebel, and Daniel Cremers.
\newblock Learning to drive using inverse reinforcement learning and deep
  {Q}-networks.
\newblock \emph{arXiv preprint arXiv:1612.03653}, 2016.

\bibitem[Stooke \& Abbeel(2019)Stooke and Abbeel]{stooke2019rlpyt}
Adam Stooke and Pieter Abbeel.
\newblock rlpyt: A research code base for deep reinforcement learning in
  pytorch.
\newblock \emph{arXiv preprint arXiv:1909.01500}, 2019.

\bibitem[Syed et~al.(2008)Syed, Bowling, and Schapire]{syed2008apprenticeship}
Umar Syed, Michael Bowling, and Robert~E Schapire.
\newblock Apprenticeship learning using linear programming.
\newblock In \emph{Proceedings of the 25th International Conference on Machine
  Learning (ICML)}, pp.\  1032--1039, 2008.

\bibitem[Wu et~al.(2020)Wu, Sun, Zhan, Yang, and Tomizuka]{wu2020efficient}
Zheng Wu, Liting Sun, Wei Zhan, Chenyu Yang, and Masayoshi Tomizuka.
\newblock Efficient sampling-based maximum entropy inverse reinforcement
  learning with application to autonomous driving.
\newblock \emph{IEEE Robotics and Automation Letters}, 5\penalty0 (4):\penalty0
  5355--5362, 2020.

\bibitem[Yang et~al.(2019)Yang, Li, and Zhang]{yang2019regularized}
Wenhao Yang, Xiang Li, and Zhihua Zhang.
\newblock A regularized approach to sparse optimal policy in reinforcement
  learning.
\newblock In \emph{Advances in Neural Information Processing Systems
  (NeurIPS)}, pp.\  5938--5948, 2019.

\bibitem[Zhang et~al.(2019)Zhang, Dai, Li, and Schuurmans]{zhang2019gendice}
Ruiyi Zhang, Bo~Dai, Lihong Li, and Dale Schuurmans.
\newblock {G}en{D}ice: Generalized offline estimation of stationary values.
\newblock In \emph{International Conference on Learning Representations}, 2019.

\bibitem[Ziebart(2010)]{ziebart2010modeling}
Brian~D Ziebart.
\newblock Modeling purposeful adaptive behavior with the principle of maximum
  causal entropy.
\newblock 2010.

\bibitem[Ziebart et~al.(2008)Ziebart, Maas, Bagnell, and
  Dey]{ziebart2008maximum}
Brian~D Ziebart, Andrew~L Maas, J~Andrew Bagnell, and Anind~K Dey.
\newblock Maximum entropy inverse reinforcement learning.
\newblock In \emph{Proceedings of the 23rd AAAI Conference on Artificial
  Intelligence}, volume~8, pp.\  1433--1438, 2008.

\end{thebibliography}
